\documentclass[letterpaper,10pt,reqno]{amsart}

\usepackage{amssymb}
\usepackage{latexsym}
\usepackage{amsbsy}
\usepackage{amsfonts}
\usepackage{hyperref}
\usepackage{graphicx}
\usepackage{enumerate}
\usepackage{enumitem}
\usepackage{scalerel}
\usepackage{color}
\usepackage{natbib}

\usepackage{tikz}

\def\marginpar#1{\ignorespaces}

\DeclareMathOperator\tr{Tr}
\DeclareMathOperator\ess{ess \, sup}

\newtheorem{theorem}{Theorem}[section]
\newtheorem{lemma}[theorem]{Lemma}

\newtheorem{corollary}[theorem]{Corollary}

\newtheorem{assumption}[theorem]{Assumption}

\newtheorem{example}[theorem]{Example}
\numberwithin{equation}{section}

\newcommand{\beq}{\begin{equation}}
\newcommand{\eeq}{\end{equation}}

\newcommand{\bal}{\begin{align}}
\newcommand{\eal}{\end{align}}
\newcommand{\bals}{\begin{align*}}
\newcommand{\eals}{\end{align*}}

\makeatother
\begin{document}
\title[Regret of $q$-learning]{Regret of exploratory policy improvement and $q$-learning}

\author[Wenpin Tang]{{Wenpin} Tang}
\address{Department of Industrial Engineering and Operations Research, Columbia University.
} \email{wt2319@columbia.edu}

\author[Xun Yu Zhou]{{Xun Yu} Zhou}
\address{Department of Industrial Engineering and Operations Research, Columbia University.
} \email{xz2574@columbia.edu}

\date{\today}

\begin{abstract}
We study the convergence of $q$-learning and related algorithms introduced by
Jia and Zhou (J. Mach. Learn. Res., 24 (2023), 161) for controlled diffusion processes.
For exploratory policy improvement, we establish exponential convergence under growth and regularity assumptions on the model parameters.
For $q$-learning,
we derive quantitative error and regret bounds
under additional assumptions on the function approximation and the associated stochastic approximation dynamics.
\end{abstract}

\maketitle


\section{Introduction}

\quad Reinforcement learning (RL) is an active subarea in machine learning,
which has successfully been applied to solve complex decision-making problems
such as playing board games (\cite{Silver16, Silver17}) and video games (\cite{Minh15}),
driving autonomously (\cite{Lev16, Mir17}),
and more recently, aligning large language models and text-to-image generative models
with human preference (\cite{Black23, Fan24, Ou22}).
RL research has predominantly focused on Markov decision processes (MDPs) in discrete time and space;
see \cite{SB18} for a detailed account of theory and applications for MDPs.

\quad \cite{WZZ20} are the first to formulate
and develop an entropy-regularized, exploratory control framework for RL with controlled diffusion processes, which is inherently in continuous
time with continuous state spaces and possibly continuous action (control) spaces.
In this framework, stochastic relaxed control is utilized to represent exploration,
capturing the notion of “trial and error” that is the core to RL.
Subsequent work aims at laying theoretical foundation for model-free RL in continuous-time
by a martingale approach (\cite{JZ221, JZ222, JZ22}),
and by policy optimization (\cite{ZTY23}). Here, by ``model-free" we mean that the underlying dynamics are diffusion processes but their coefficients along with the reward functions are unknown.
The key insight of \cite{JZ221, JZ222, JZ22} is that one can derive learning objectives
from the martingale structure underpinning the continuous-time RL.
The theoretical results in those papers naturally lead to various ``model-free" algorithms for general RL tasks, in the sense that they learn optimal policies directly without attempting to learn/estimate the model parameters.
Many of these algorithms recover existing RL algorithms for MDPs
that were often proposed in a heuristic and/or ad hoc manner.
However, convergence and regret analysis of the algorithms, which has occupied a central stage in the RL study for MDPs, is still lacking for the diffusion counterpart.
To our best knowledge, the only works that carry out a {\it model-free} convergence analysis and derive sublinear regrets are
\cite{HJZ2024} for continuous-time mean--variance portfolio selection and
\cite{huang2024sublinear,huang2025data} for a class of stochastic linear--quadratic (LQ) control problems,
all of which apply/adapt the policy gradient algorithms developed  in \cite{JZ221} and exploit heavily the special structures of the problems.

\quad The purpose of the paper is to fill this gap
by providing a quantitative analysis of (little) {\em $q$-learning} introduced in \cite{JZ22} and related algorithms for generally nonlinear RL problems. While (big) $Q$-learning is a key method for discrete-time MDP RL, the $Q$-function collapses in continuous time as it no longer depends on actions when the time step is infinitesimally small.  \cite{JZ22} proposes the notion of the $q$-function, which is the first-order derivative of the $Q$-function
with respect to time discretization. The $q$-function is entirely a continuous-time notion, which is closely related to the (generalized) Hamiltonian. The significance of this function lies in several aspects: 1) it is the function that needs to be learned in order to learn the optimal policies, rather than the individual functions appearing in the model; 2) it is the Gibbs exponent that can be used to improve the current policy; and 3) it is {\it learnable} (through an argument using It\^o's formula) by observable/computable  data including the temporal differences and the reward signals.
This leads to $q$-learning,
which alternates between stochastic approximation (to learn the value functions and the $q$-functions)
and policy iteration/improvement (to update and improve the policies). The general $q$-learning theory and algorithms have already been applied to various specific problems, including in particular diffusion models for generative AI (\cite{gao2024reward, ZZT24}).

\quad The $q$-learning has two components:
\begin{equation*}
\mbox{ Learning values  and policies } + \mbox{ Policy iteration (PI)},
\end{equation*}
where ``values" include those of the value functions and $q$-functions and ``policies" are encoded by the $q$-functions.
In the classical control setting,
PI or Howard's algorithm (\cite{Bell57, How60}) improves the policies gradually
in order to approximate the optimal one.
So it is also called policy improvement; see \cite{KSS20,JZ22, TWZ24} for recent studies.
In the RL setting,
an entropy regularizer is added to encourage exploration,
which gives rise to the
exploratory control problem (\cite{TZZ21}).
When the model parameters are known
(and the value function can be accessed),
one can simply use a version of PI for the exploratory control problem.
We call it {\em exploratory policy improvement},
which can be viewed as a model-based RL approach;
see \cite{MWZ24, SSZ24, TWZ24} for recent development.
The model-free RL approach,
on the other hand,
consists of learning directly the value function and the optimal policy, {\it without} estimating the model parameters.

\quad Our contribution is to provide a quantitative convergence analysis for both
the model-based exploratory policy improvement and the model-free $q$-learning,
while the former serves as a technical prerequisite for the latter.
Specially, we prove:
\begin{itemize}[itemsep = 3 pt]
\item
the exponential convergence of exploratory policy improvement (Theorem \ref{thm:cvrate});
\item
an explicit error bound of $q$-learning, depending on the regularity of the model parameters
and the learning rate
(Theorems \ref{thm:main5} and \ref{thm:final}).
\end{itemize}
As an intermediate step to study the $q$-learning algorithm,
we introduce {\em semi-$q$-learning} in which
the value functions and the reward functions are assumed to be known;
so we only need to learn the $q$-functions.
We establish a bound on the value approximation
in terms of the $q$-function approximation (Theorem \ref{thm:main2}),
which is crucial for the convergence analysis of (semi-)$q$-learning.

\quad Our proofs rely on both probabilistic and analytical arguments,
including backward stochastic differential equations (BSDEs),
partial differential equations (PDEs),
perturbation analysis,
and stochastic approximation.
In particular, we provide the first general quantitative framework
for the mode-free $q$-learning based on stochastic approximation,
which allows us to perform further convergence and regret analysis.
We hope that this work will trigger and motivate further research in  continuous-time RL,
and especially in its emerging application to diffusion model alignment (\cite{gao2024reward, ZZT24, ZZT25}, \cite{Sheng25}, \cite{HTZ26}).

\quad The remainder of the paper is organized as follows.
In Section \ref{sc2}, we provide background on the exploratory control problem and $q$-learning.
In Section \ref{sc3}, we study exploratory policy improvement,
and the $q$-learning algorithm is considered in Section \ref{sc4}.
We conclude with Section \ref{sc5}.

\section{Background}
\label{sc2}

\quad In this section, we provide background on
continuous-time RL, and in particular
$q$-learning, for controlled diffusion processes.
We first introduce notations that will be used throughout.
\begin{itemize}[itemsep = 3 pt]
\item
$\mathbb{R}$ is the set of real numbers.
For $x, y \in \mathbb{R}$,
$x \wedge y$ (resp. $x \vee y$) is the smaller (resp. larger) number of $x$ and $y$.
\item
For  a vector $x$, $|x|$ is the Euclidean norm of $x$;
for $\mathcal{A} \subset \mathbb{R}^d$, $|\mathcal{A}|$ is the volume of $\mathcal{A}$.
\item
For a matrix $A$, $A^T$ is the transpose of $A$, $\tr(A)$ is the trace of $A$,
and $|A|$ is the spectral norm of $A$.
\item
For a function $f$ on $X$, $|f|_\infty:= \sup_{X}|f(x)|$ denotes its sup-norm,
and $|f|_{\mathbb{L}^1(X)}:= \int_X |f(x)|dx$ is the $\mathbb{L}^1$-norm of $f$.
\item
For $f:[0,\infty) \times \mathbb{R}^d \ni (t,x) \to \mathbb{R}$,
$\partial_t f$ is its (partial) derivative in $t$,
and $\nabla f = \left(\frac{\partial f}{\partial x_i}\right)_i$ and $\nabla^2 f = \left(\frac{\partial f}{\partial x_i \partial x_j}\right)_{i,j}$ are its gradient and  Hessian in $x$ respectively.
\item
For two probability density functions $p(\cdot)$ and $q(\cdot)$,
$d_{TV}(p(\cdot), q(\cdot)):= \sup_A |\int_Ap(a)da -\int_Aq(a)da|$ is the total variation distance between $p(\cdot)$ and $q(\cdot)$.
\item
For a $\sigma$-field $\mathcal{F}$, $\mathbb{L}^2(\mathcal{F})$ is the set of $\mathcal{F}$-measurable square-integrable random variables.
\item For any $\theta\in \mathbb{R}$ and $0 \le t \le T$, $\mathbb{H}^\theta_t:=\{\mbox{predictable processes }(x_s, \, t \le s \le T)$ satisfying
$$\mathbb{E}\int_t^T e^{\theta s} |x_s|^2 ds < \infty\},$$
with the $\mathbb{H}^\theta_t$-norm
$|x|_{\mathbb{H}^\theta_t}:= (\mathbb{E}\int_t^T e^{\theta s} |x_s|^2 ds)^{\frac{1}{2}}$.
\item
$a = \mathcal{O}(b)$ or $a \lesssim b$ means that  $a/b$ is bounded as some problem parameter tends to $0$ or $\infty$.
\item We use $C$ for a generic constant whose values may change from line to line.
\end{itemize}

\subsection{Classical control problem}
Let $(\Omega, \mathcal{F}, \mathbb{P}, \{\mathcal{F}_t\}_{t \ge 0})$ be a filtered probability space on which we define a $d'$-dimensional $\mathcal{F}_t$-adapted Brownian motion $(W_t, \, t \ge 0)$.
Let $\mathcal{A}$ be a generic action space,
and $u = (u_t, \, 0 \le t  \le T)$ be a control which is an adapted process taking values in $\mathcal{A}$.

\quad The stochastic control problem is to control the state variable $X_t \in \mathbb{R}^d$,
whose dynamic is governed by the (controlled) stochastic differential equation (SDE):
\begin{equation}
\label{eq:classicalS}
dX^u_t = b(t,X^u_t, u_t) dt + \sigma(t,X^u_t, u_t) dW_t,
\end{equation}
where $b: \mathbb{R}_+ \times \mathbb{R}^d \times \mathcal{A} \to \mathbb{R}^d$ is the drift coefficient,
and $\sigma: \mathbb{R}_+ \times \mathbb{R}^d \times \mathcal{A} \to \mathbb{R}^{d \times d'}$ is the covariance matrix (or diffusion coefficient) of the state variable.
Here the superscript `$u$' in $X^u_t$ emphasizes the dependence of the state variable on  control $u$.
The goal of the control problem is to maximize the total discounted reward (or objective functional), leading to the (optimal) value function:
\begin{equation}
\label{eq:classical}
\mathring{J}(t,x):=
 \sup_{u} \mathbb{E}\left[ \int_t^{T} e^{-\beta (s-t)}r(s, X^u_s, u_s) ds + e^{-\beta (T-t)}h(X^u_T)\bigg| X^u_t = x\right],
\end{equation}
where
$r: \mathbb{R}_+ \times \mathbb{R}^d \times \mathcal{A} \to \mathbb{R}$ is the running reward,
$h: \mathbb{R}^d \to \mathbb{R}$ is the terminal reward,
and $\beta > 0$ is the discount factor.

\quad In the classical setting, the functional forms of $r,h,b,\sigma$ are known.
The optimal control is generally represented as a deterministic mapping from the current time and state to the action space:
$u^*_t = u^*(t, X_t^*)$.
The mapping $u^*$ is called an optimal feedback control {\it policy},
and the corresponding optimally controlled process $(X_t^*, \, t \ge 0)$ satisfies the following SDE:
\begin{equation*}
dX^{*}_t = b(t,X_t^{*}, u^{*}(t, X^{*}_t))dt +  \sigma(t,X_t^{*}, u^{*}(t, X^{*}_t))dW_t,
\end{equation*}
provided that it is well-posed. Optimal policies are derived by powerful approaches such as dynamic programming or maximum principle;
see \cite{FS06, YZ99} for detailed accounts of the classical stochastic control theory.

\subsection{Exploratory control problem for RL}
When
the model parameters are unknown,\footnote{In many practical problems, the parameters $b, \sigma$ are unknown while the rewards $r, h$ may be specified in advance. The problem where $r, h$ are unknown and need to be learned from the observed (optimal) actions is referred to as {\em inverse reinforcement learning} (\cite{Rus98, NR00}).}
the RL approach explores the unknown environment and learns optimal controls through repeated trials and errors. This calls for an essentially and drastically different school of thoughts from that of the classical control theory.

\quad Inspired by this, \cite{WZZ20} models exploration by
a probability distribution of controls $\pi = (\pi_t(\cdot), \, 0 \le t \le T)$
over the action space $\mathcal{A}$ from which each trial is sampled.
The exploratory state process is:
\begin{equation}
\label{eq:Xpi}
dX^\pi_t = \widetilde{b}(t,X^\pi_t, \pi_t) dt + \widetilde{\sigma}(t,X^\pi_t, \pi_t) d\widetilde W_t,
\end{equation}
where $(\widetilde W_t, \, t \ge 0)$ is a $d$-dimensional $\mathcal{F}_t$-adapted Brownian motion, and the coefficients $\widetilde{b}(\cdot, \cdot, \cdot)$ and $\widetilde{\sigma}(\cdot, \cdot, \cdot)$ are defined by
\begin{equation}
\label{eq:bsigtil}
 \widetilde{b}(t,x, \pi):=\int_{\mathcal{A}} b(t,x,a) \pi(a) da \quad \mbox{and} \quad
  \widetilde{\sigma}(t,x, \pi):= \left(\int_{\mathcal{A}}\sigma(t,x,a) \sigma(t,x,a)^T \pi(a) da \right)^{\frac{1}{2}},
\end{equation}
for $(t, x,\pi)\in \mathbb{R}_+ \times \mathbb{R}^d\times \mathcal{P}(\mathcal{A})$,
with $\mathcal{P}(\mathcal{A})$ being the set of absolutely continuous probability density functions on $\mathcal{A}$.
The distributional control $\pi = (\pi_t(\cdot), \, 0 \le t \le T)$ is also known as the relaxed control in the control literature.
The objective is to maximize the entropy-regularized problem:
\begin{equation}
\label{eq:Jstargam}
\begin{aligned}
& J^*(t, x) \\
&:= \sup_\pi J(t,x; \pi) \\
& := \sup_\pi \mathbb{E} \bigg[ \int_t^T  e^{-\beta(s-t)} \int_\mathcal{A} [r(s, X^\pi_s, a) - \gamma \log \pi_s(a)] \pi_s(a)da \, ds
+ e^{-\beta(T -t)}h(X^\pi_T) \bigg| X^\pi_t = x \bigg],
\end{aligned}
\end{equation}
where $J(\cdot,\cdot; \pi)$ denotes the value function under the control $\pi(\cdot)$,
and $\gamma > 0$ is the (temperature) parameter representing the weight on  exploration (so a greater $\gamma$ encourages more exploration).

\quad Let
\begin{equation*}
H(t,x,a,p,q): = b(t,x,a) \, p + \frac{1}{2} \tr \left(\sigma(t,x,a) \sigma^T(t,x,a) q \right) + r(t,x,a)
\end{equation*}
be the generalized Hamiltonian.
According to \cite{TZZ21}, Section 5,
under suitable conditions on the parameters $(r,h,b,\sigma)$,
$J^*$ is the solution to the {\it exploratory} Hamilton--Jacobi--Bellman (HJB) equation:
\begin{equation}
\label{eq:expHJB}
\left\{ \begin{array}{lcl}
\partial_t J^* + \gamma \log \int_{\mathcal{A}} \exp \left(\frac{1}{\gamma} H(t,x,a, \nabla J^*, \nabla^2 J^*) \right)da -
\beta J^* = 0, \\
J^*(T,x) = h(x).
\end{array}\right.
\end{equation}
The corresponding optimal (exploratory) feedback control policy is the Gibbs measure:
\begin{equation}
\label{eq:pistargam}
\pi^*(\cdot \,|\, t,x) \propto \exp\left(\frac{1}{\gamma} H(t,x,\cdot, \nabla J^*, \nabla^2 J^*) \right) .
\end{equation}
See \cite{TZZ21, Zhou23} for further development on the exploratory control problem \eqref{eq:Jstargam},
and its HJB equation \eqref{eq:expHJB}.

\quad In the remainder of this paper,
we assume that control only appears in the drift term.\footnote{Our analysis relies on the BSDE representation of the exploratory control problem, where control being only in the drift term corresponds to a semi-linear PDE.
The case where control also appears in the diffusion term requires a more complicated stochastic representation of fully nonlinear PDEs; see \cite{CSTV, Peng07}.
This is also related to the second-order risk adjustment of stochastic maximum principle
when control enters into the diffusion term;
see \cite{Peng90} and \cite{YZ99}, Chapter 3.
We leave this case for future work.}
That is, $\sigma(t,x,a) = \sigma(t,x)$ is independent of the control;
so the resulting exploratory state process is governed by
\begin{equation}
\label{eq:Xpidrift}
dX^\pi_t = \widetilde{b}(t,X^\pi_t, \pi_t) dt + \sigma(t,X^\pi_t) d\widetilde W_t.
\end{equation}

\subsection{$q$-learning}\label{qlearning}
The main task of RL is to solve the entropy-regularized problem \eqref{eq:Jstargam} by learning its
optimal policy \eqref{eq:pistargam}. 

\quad The first idea, based on the policy improvement theorems in \cite{JZ22, WZ20},
is to approximate $\pi^*(\cdot \,|\, \cdot,\cdot)$
by a sequence of controls $\pi^n(\cdot \,|\, \cdot,\cdot)$
such that the corresponding value functions $J^n$ are nondecreasing,
i.e. $J^{n+1}(t,x) \ge J^n(t,x)$ for all $n$ and all $(t,x)$.
The construction of the policies $\pi^n(\cdot \,|\, \cdot,\cdot)$
is through the iterations:
\begin{equation}
\label{eq:sampling11}
\pi^{n+1}(\cdot \,|\, t,x) \propto \exp\left(\frac{1}{\gamma} H(t,x,\cdot, \nabla J^n, \nabla^2 J^n) \right).
\end{equation}
This approach is referred to as ``exploratory policy improvement" aiming to approximate the optimal policy.
However,
the functions $(b, \sigma,r,h)$ are unknown in the RL setting,
whereas the sampling \eqref{eq:sampling11} requires an oracle access to them.
Moreover,
the derivatives of the value functions $(\nabla J^n, \nabla^2 J^n)$
are generally not easy to evaluate even if one can learn
$J^n$ through policy evaluation,
which yields further difficulty in updating the policy.

\quad The second idea, as proposed in \cite{JZ22}, is to learn {\it directly} a term equivalent to $H(t,x,a, \nabla J^n, \nabla^2 J^n)$ for the purpose of sampling
\eqref{eq:sampling11}. That term is called the (little) $q$-function
(analogous to the $Q$-function in the classical discrete-time RL for MDPs). Specifically,
given a control $\pi(\cdot)$,
the $q$-function is
\begin{equation}
\label{eq:qfunc}
\begin{aligned}
q(t,x,a; \pi) : &= \frac{\partial J}{\partial t}(t,x; \pi) + H\left(t,x,a, \nabla J(t,x; \pi), \nabla^2 J(t,x; \pi)\right) - \beta J(t,x; \pi) \\
& = b(t,x,a) \nabla J(t,x; \pi) + r(t,x,a) + \frac{\partial J}{\partial t}(t,x; \pi)  \\
& \qquad \qquad \qquad \qquad \qquad \qquad+ \frac{1}{2} \tr\left(\sigma^2(t,x) \nabla^2 J(t,x; \pi) \right)
- \beta J(t,x; \pi).
\end{aligned}
\end{equation}
Denote by $a^\pi = (a^\pi_t, \, 0 \le t \le T)$ a control realization from sampling the policy $\pi(\cdot)$. 
Conceptually, the action process $a^\pi$ should be sampled independently over (continuous) time.
However, as pointed out in \cite[Remark 2.1]{STZ24} and \cite[Section 3]{BT24}
the resulting process $a^\pi$ is not necessarily progressively measurable.
This was also observed earlier in \cite[Section 19.5]{Sto87} and \cite[Example 1.2.5]{Kal80}.
One remedy is to consider the discretized version:
\begin{equation*}
a^{\pi}_t = \sum_{k = 0}^{N-1} a^{\pi}_{t_k} 1(t_{k} \le t < t_{k+1}),
\quad \mbox{for a partition } 0=t_0 < t_1 < \cdots < t_N = T,
\end{equation*}
which is termed a ``discretely sampled process (DSPs)" in \cite{jia2026accuracy}. 
Such defined $(a_t^{\pi^\phi})$ is measurable;
so all the integrals involving $a^\pi$ are well-defined. \cite{jia2025erratum}, an erratum of \cite{JZ22}, restate and reprove all the $q$-learning results of the latter in terms of DSPs. Note that DSPs are exactly what are actually  implemented in any real application;  so the theory and implementation/application align here. To simplify notation and presentation, henceforth we assume that sampling of $\pi(\cdot)$ is always in the sense of DSPs.
See also Footnote \ref{bias} for a remark on the discretization error arising from DSPs.

\quad The key takeaway in \cite{JZ22} is to observe the fact that
\begin{equation}
e^{-\beta s} J(s, X^\pi_s; \pi) + \int_t^s e^{-\beta u} \left[ r(u, X^\pi_u, a^\pi_u) - q(u, X^\pi_u, a^\pi_u; \pi) \right] du,
 \quad t \le s \le T,
\end{equation}
is a martingale.
By parametrizing $(J(t, x; \pi), q(t, x, a; \pi))$ with $\{(J^\theta(t,x), q^\phi(t,x,a))\}_{\theta, \phi}$,
the minimization of the martingale loss function motivates the following optimization problem:
\begin{equation}
\label{eq:offopt11}
\min_{(\theta, \phi)} \mathbb{E}\left[\int_0^T \left\{e^{-\beta (T-t)} h(X_T^{\pi}) - J^\theta(t,X_t^{\pi})
+ \int_t^T e^{-\beta(s-t)} [r(s, X^{\pi}_s, a^{\pi}_s) - q^\phi(s, X^{\pi}_s, a^{\pi}_s)] ds \right\}^2 dt \right].
\end{equation}
Problem \eqref{eq:offopt11} can be solved by stochastic gradient descent (SGD):
\begin{equation*}
\begin{aligned}
& \theta \leftarrow \theta + \alpha_\theta \int_0^T \frac{\partial J^\theta}{\partial \theta}(t, X_t^\pi) G_{t:T}dt, \\
& \phi \leftarrow \phi + \alpha_\phi \int_0^T \int_t^T e^{-\beta(s-t)} \frac{\partial q^\phi}{\partial \phi}(s, X_s^\pi, a_s^\pi) ds G_{t:T}dt,
\end{aligned}
\end{equation*}
where $G_{t:T}: = e^{-\beta (T-t)} h(X_T^{\pi}) - J^\theta(t,X_t^{\pi})
+ \int_t^T e^{-\beta(s-t)} [r(s, X^{\pi}_s, a^{\pi}_s) - q^\phi(s, X^{\pi}_s, a^{\pi}_s)] ds$,
and $\alpha_\theta, \alpha_\phi$ are user-defined learning rates. Note that in actual implementation of this procedure one needs to take DSPs (described earlier) for actions while observing the resulting states and reward signals at the corresponding discretized time points and calculating the integrals as summations using the same grid points. Thus, the update rule is entirely data driven without having to know/learn any model parameters.

\quad Now we apply the above procedure to each policy iteration.
Since only the term $H\left(t,x,a, \nabla J, \nabla^2 J\right)$ in \eqref{eq:qfunc}
involves the control variable $a$,
the policy update \eqref{eq:sampling11} is equivalent to
\begin{equation*}
\pi^{n+1}(\cdot \,|\, t,x) \propto \exp\left(\frac{1}{\gamma} q(t,x,\cdot; \pi^n) \right).
\end{equation*}
Using the SGD step in each policy update yields the
$q$-learning algorithm:
start with some $(\theta_1, \phi_1)$
and a control policy $\pi^{1}(\cdot \,|\, \cdot,\cdot)$,
and for $n \ge 1$,
\begin{enumerate}[itemsep = 3 pt]
\item
Update
\begin{equation}
\label{eq:SGD0}
\begin{aligned}
& \theta_{n+1} = \theta_n + \alpha_{\theta,n} \int_0^T \frac{\partial J^\theta}{\partial \theta}_{|\theta = \theta_n}(t, X_t^{\pi^{n}}) G^n_{t:T}dt, \\
&\phi_{n+1} = \phi_n + \alpha_{\phi,n} \int_0^T \int_t^T e^{-\beta(s-t)} \frac{\partial q^\phi}{\partial \phi}_{|\phi = \phi_n}(s, X_s^{\pi^{n}}, a_s^{\pi^{n}}) ds G^n_{t:T}dt,
\end{aligned}
\end{equation}
where $$G^n_{t:T}: = e^{-\beta (T-t)} h(X_T^{\pi^{n}}) - J^{\theta_n}(t,X_t^{\pi^{n}})
+ \int_t^T e^{-\beta(s-t)} [r(s, X^{\pi^{n}}_s, a^{\pi^{n}}_s) - q^{\phi_n}(s, X^{\pi^{n}}_s, a^{\pi^{n}}_s)] ds.$$
\item
Sample
\begin{equation}
\label{eq:sample0}
\pi^{n+1}(\cdot\,|\, t,x) \propto \exp\left(\frac{1}{\gamma} q^{\phi_{n+1}}(t,x,\cdot) \right).
\end{equation}
\end{enumerate}

\section{Convergence of exploratory policy improvement}
\label{sc3}

\quad In this section,
we study the convergence of exploratory policy improvement
by assuming
an oracle access to some of the model parameters
as well as the value functions.
This allows us to understand
how policy improvement itself works,
and our result (Theorem \ref{thm:cvrate})
shows that it is exponentially fast.
The arguments will then be used for
the subsequent (model-free) analysis of the $q$-learning algorithm \eqref{eq:SGD0}--\eqref{eq:sample0};
so this section is also a technical preparation for later development.

\quad Recall \eqref{eq:Xpidrift} that defines the exploratory state process $X^\pi$.
The exploratory policy improvement
starts with
a control policy $\pi^{1}(\cdot \,|\, \cdot,\cdot)$,
and for $n \ge 1$,
{\small \begin{equation}
\begin{aligned}
\label{eq:Jngam}
J^n(t, x): = \mathbb{E} \bigg[ \int_t^T  e^{-\beta(s-t)} \int_\mathcal{A} [r(s, X^{\pi^n}_s, a) -  \gamma \log \pi^n&(a|s,X^{\pi^n}_s)]  \pi^n(a|s,X^{\pi^n}_s)da \, ds \\
&+ e^{-\beta(T -t)}h(X^{\pi^n}_T) \bigg| X^{\pi^n}_t = x \bigg].
\end{aligned}
\end{equation}}
Sample
\begin{equation}
\label{eq:pingam}
\begin{aligned}
\pi^{n+1}(\cdot \,|\, t,x) & \propto \exp\left(\frac{1}{\gamma} H(t,x,\cdot, \nabla J^n(t,x), \nabla^2 J^n(t,x)) \right)  \\
& \propto \exp \left(\frac{1}{\gamma} \left( b(t,x,\cdot) \nabla J^n(t,x) + r(t,x,\cdot) \right) \right) ,
\end{aligned}
\end{equation}
where we omit to write out the dependence of $\gamma$ in the expressions of  $J^n$ and  $\pi^n$
for simplicity. Note that
since $\sigma(t,x,a) = \sigma(t,x)$,
only the terms $b(t,x,a) \nabla J^n(t,x)$ and $ r(t,x,a)$
in the Hamiltonian
involve the control $a$.
This yields the (simpler) policy update \eqref{eq:pingam}.

\quad As previously mentioned, we assume that
the value function $J^n$ defined in \eqref{eq:Jngam}
(and hence its gradient $\nabla J^n$)
and the functions $b,r$ are all accessed.
This corresponds to the model-based RL approach;
so we do not need to estimate these functions
nor any of their combinations.
Here, our goal is to bound
\begin{equation*}
\left|J^n(t,x) - J^*(t,x)\right|, \quad \mbox{in terms of } n \mbox{ (and } \gamma),
\end{equation*}
where $J^*$ is the optimal value function of the exploratory control problem \eqref{eq:Jstargam}.

The convergence of $J^n$ is proved in \cite{HWZ22}, but with no convergence rate given.
Two concurrent papers (\citealp{MWZ24, TWZ24}) study the convergence rate of
{\it model-based} policy iteration \eqref{eq:pingam}.
Their results, the exponential convergence of $J^n$, are similar to Theorem \ref{thm:cvrate} below.
\cite{TWZ24}, as \cite{HWZ22},  is based on a PDE analysis,
where the model parameters are assumed to be H\"older smooth.
This yields Schauder's estimates
and, hence, a stronger convergence in H\"older spaces.
However, these PDE methods do not track explicitly the dependence
on the level of exploration $\gamma$.
\cite{MWZ24} rely on Malliavin calculus,
where they assume the model parameters to be smooth and bounded as in our paper below.
The dependence on $\gamma$, which depends on higher-order derivatives of the Gibbs map, is not transparent there as well.
Nevertheless, our BSDE approach differs from \cite{MWZ24, TWZ24},
and our result provides an explicit dependence on the exploration $\gamma$,
which is important in efficient sampling of the policies \eqref{eq:pingam}
but is missing in the aforementioned work.\footnote{It is worth noting  that \cite{MWZ24, TWZ24} prove the  convergence of $J^n$ to $J^*$ in H\"older norms, which is stronger than our pointwise convergence in Theorem \ref{thm:cvrate}.
Nevertheless, our goal is to derive the regret in the reinforcement learning setting,
which is the total reward loss $\sum_{k = 1}^n (J^* - J^k)$.
So the pointwise convergence is sufficient for this purpose.}

\quad Our analysis for exploratory policy improvement follows the general BSDE framework developed in \cite{KSS20}.
In particular, several stability and contraction estimates are adaptations of arguments from \cite{KSS20} to the exploratory-control setting considered here.
The main new ingredients are the entropy-regularized Gibbs policies and the resulting  nonlinear BSDE generators.
We stress again that the result in this section is a preparation for solving the ``learning" problem in Section \ref{sc4} that is at the heart of model-free $q$-learning,
where the treatment of approximation errors arising from $q$-learning
was considered neither in \cite{KSS20}, nor in \cite{MWZ24, TWZ24}.\footnote{In this sense, \cite{KSS20,MWZ24, TWZ24} are ``numerical" papers solving classical stochastic control problems, while ours is a ``learning" paper solving the model-free counterparts.}

\quad To proceed,
we make the following (mild) assumptions.
\begin{assumption}
\label{assump}
$~$
\begin{enumerate}
\item[(i)]
$\mathcal{A}$ is a compact set.
\item[(ii)]
$r,h,b,\sigma$ are continuous in their respective arguments.
\item[(iii)]
There exist $\overline{b}, \overline{\sigma}, \overline{r} > 0$ such that for any $(t,x,a)$,
\begin{equation*}
|b(t,x,a)| \le \overline{b}, \quad \frac{1}{\overline{\sigma}} \le |\sigma(t,x)| \le \overline{\sigma} \quad \mbox{and} \quad |r(t,x,a)| \le \overline{r}.
\end{equation*}
\item[(iv)]
There exists $K > 0$ such that for any $(t,x,y,a)$,
\begin{equation*}
|b(t,x,a) - b(t,y, a)| + |\sigma(t,x) - \sigma(t,y)| + |r(t,x,a) - r(t,y,a)| + |h(x) - h(y)|
\le K|x-y|.
\end{equation*}
\item[(v)]
$r, b, \sigma$ belong to $\mathcal{C}^{\frac{\alpha}{2},\alpha}$ in $(t,x)$ uniformly in $a$,
and $h \in \mathcal{C}^{2+\alpha}(\mathbb R^d)$ for some $\alpha > 0$.
\item[(vi)]
For the initial policy, $\log \pi^1$ belongs to $\mathcal{C}^{\frac{\alpha}{2},\alpha}$ in $(t,x)$ uniformly in $a$.
\end{enumerate}
\end{assumption}
\quad We emphasize that \cite{JZ22} require only the weaker linear-growth condition on $b, \sigma$ for their policy improvement and $q$-learning framework.
Here the stronger boundedness assumption (iii) is imposed to obtain further quantitative convergence and regret estimates,
especially for deriving Lipschitz and contraction estimates of entropy-regularized Gibbs policies.

\quad Our result is stated as follows.
\begin{theorem}
\label{thm:cvrate}
Let Assumption \ref{assump} hold,
and fix $\eta \in (0,1)$.
There exist $L, C> 0$ (independent of $\gamma$ and $n$) such that
\begin{equation}
\label{eq:expcv}
|J^*(t,x) - J^n(t,x)|^2 \le C \eta^n e^{\theta(\gamma) (T -t)},
\end{equation}
where
$\theta(\gamma):=\beta + (1+\eta^{-1}) L^2\left( 1+ e^{\frac{L}{\gamma}} + \frac{1}{\gamma} e^{\frac{L}{\gamma}} \right)^2$.
\end{theorem}

We make several remarks
before proving the theorem.
First,
the bound \eqref{eq:expcv} implies that the exploratory policy constructed by \eqref{eq:pingam}
converges exponentially in $n$.
The bound depends explicitly
on the level of exploration $\gamma$ via the term $e^{\theta(\gamma)(T-t)}$.
Observe that $\theta(\gamma)$ decreases in $\gamma$,
which suggests to take a large $\gamma$ for a faster convergence.
In fact,
there are two advantages of choosing a large $\gamma$:
\begin{itemize}[itemsep = 3 pt]
\item
Larger $\gamma$ induces faster convergence of the policy improvement (Theorem \ref{thm:cvrate}).
\item
Larger $\gamma$ facilitates Markov chain Monte Carlo (MCMC) sampling of the policy $\pi^n(\cdot \,|\, t,x)$
(\cite{SC22}, Chapter 4 and \cite{MSTW22}, Section 2.6).\footnote{In general, the MCMC sampler of $\pi^n(\cdot \,|\, t,x)$ converges at a rate of $\exp(- e^{-\frac{1}{\gamma}} t)$
as the level of exploration $\gamma \to 0^+$.
This implies for each iteration $n$, sampling $\pi^n(\cdot \,|\, t,x)$ requires $\mathcal{O}(e^{\frac{1}{\gamma}})$ time complexity, which is exponentially large if $\gamma$ is chosen to be small.}
\end{itemize}

\quad Second, recall that $\mathring{J}$ is the optimal value function of
the original, classical control problem \eqref{eq:classical}.
It follows from \cite[Theorem 3.7]{TZZ21} that
under suitable conditions on the parameters,\footnote{The results in \cite[Section 3]{TZZ21} imply that for $(t,x) \in [0,T] \times B_r$,
$|J^*(t,x) - \mathring{J}(t,x)|$ is controlled by the Laplace approximation error:
\begin{equation*}
\sup_{t, x,p,q}\left|\gamma \log\int_{\mathcal{A}} \exp\left(\frac{1}{\gamma} H(t,x,a,p,q) \right)da - \sup_{a \in \mathcal{A}} H(t,x,a,p,q)\right|,
\end{equation*}
where the supremum is over all $t \in [0,T]$ and all $x,p,q$ in large balls whose radii depend on $r$.
This bound is problem-dependent,
which echoes the observation in \cite{SSZ24}.
Under suitable conditions (e.g., \citealp[Corollary 4.7]{TZZ21} or when $H$ is uniformly concave in $a$ in our problem),
the above bound is of order $\gamma \ln(1 /\gamma)$ as $\gamma \to 0^+$,
which gives \eqref{eq:bias}.
Note that the bound $\gamma \ln(1/\gamma)$ is universal for a class of problems,
but it is possible to get better rates under specific problem structures.}
\begin{equation}
\label{eq:bias}
|J^*(t,x) - \mathring{J}(t,x)| \le C \gamma \ln(1 /\gamma), \quad \mbox{as } \gamma \to 0^+.
\end{equation}
Combining \eqref{eq:expcv} and \eqref{eq:bias} yields
\begin{equation}
\label{eq:ncirc}
|J^n(t,x) -\mathring{J}(t,x) | \lesssim \underbrace{\gamma \ln(1/\gamma)}_{\tiny \mbox{bias}} + \underbrace{\eta^n e^{\theta(\gamma)(T-t)}}_{\tiny \mbox{policy improvement error}},
\end{equation}
giving rise to a tradeoff between bias (due to exploration) and policy improvement error
in $\gamma$.
Heuristically, minimizing the right side of \eqref{eq:ncirc} over $\gamma$ leads to the equation
$$\frac{\ln (\gamma) + 1}{ \theta'(\gamma)} e^{-\theta(\gamma)(T-t)} = C\eta^n (T-t)$$
where $C$ is a constant independent of $\gamma$ and $n$. This suggests that we could change the temperature parameter $\gamma$ over iterations, specifically an {\em exploratory annealing} $\gamma = \gamma_n \downarrow 0$ (as $n \to \infty$), to get a shaper bound for $|J^n(t,x) -\mathring{J}(t,x)|$. Moreover, we can show (via an argument similar to that in the proof of Theorem  \ref{thm:main2}) that
\begin{equation}
\label{eq:annealing}
|J^n(t,x) -\mathring{J}(t,x) | \lesssim \sum_{k = 1}^n \eta^{n-k} \gamma_k.
\end{equation}
Thus, $\gamma_n$ needs to be fast-decaying so that the right side of \eqref{eq:annealing} converges to zero
(a sufficient condition is $\sum_k \gamma_k < \infty$).
On the other hand, MCMC sampling $\pi^{n}(\cdot \,|\, t,x)$ in each iteration has $\mathcal{O}(e^{\frac{1}{\gamma_n}})$ time complexity,
which requires $\gamma_n$ to decay slowly to avoid the curse of dimensionality.
So the annealing will benefit the convergence of policy iteration yet
at the cost of sampling hardness.
For this reason, in this paper we focus on the {\em fixed} level of exploration $\gamma > 0$,
leaving the study of exploratory annealing for future work.

\subsection{Auxiliary results}

This subsection collects a few useful results for proving Theorem \ref{thm:cvrate}.
The next lemma specifies the PDE satisfied by the value function $J^n$.

\begin{lemma}
\label{lem:PEFK}
Under Assumption \ref{assump},
$J^n$ defined by \eqref{eq:Jngam} is the unique classical solution to
\begin{equation}
\label{eq:JngamPDE}
\left\{ \begin{array}{lcl}
\partial_t J+ \int_\mathcal{A} \left(H(t,x,a, \nabla J, \nabla^2 J) - \gamma \log \pi^n(a \,|\, t,x)  \right) \pi^n(a \,|\, t,x) \, da - \beta J = 0, \\
J(T,x) = h(x).
\end{array}\right.
\end{equation}
\end{lemma}
\begin{proof}
We first prove that under Assumption \ref{assump},
the PDE \eqref{eq:JngamPDE} has a unique classical solution
\begin{equation*}
\bar J^n \in  \mathcal{C}^{1+\frac{\alpha}{2},2+\alpha}([0,T)\times\mathbb R^d)
\cap \mathcal{C}([0,T]\times\mathbb R^d).
\end{equation*}
Indeed,  for a given policy $\pi^n(\cdot\,|\,t,x)$, the PDE \eqref{eq:JngamPDE} is linear parabolic:
\begin{equation*}
\begin{aligned}
&\partial_t J + \left(\int_\mathcal{A} b(t,x,a)\pi^n(a\,|\,t,x)da\right) \nabla J
+ \frac{1}{2} \tr\left(\sigma \sigma^T(t,x) \nabla^2 J \right) \\
& \quad + \int_\mathcal{A} r(t,x,a)\pi^n(a\,|\,t,x)da - \gamma \int_\mathcal{A}\log \pi^n(a\,|\,t,x)\pi^n(a\,|\,t,x)da - \beta J = 0.
\end{aligned}
\end{equation*}
Since $\sigma$ is positive definite (Assumption \ref{assump}-(iii))
and the coefficients are H\"{o}lder ((v)--(vi)),
Schauder's theory (see e.g., \cite{Fri64}) implies this equation when $n=1$ has a unique solution $\bar J^1 \in \mathcal{C}^{1+\frac{\alpha}{2},2+\alpha}([0,T)\times\mathbb R^d)
\cap \mathcal{C}([0,T]\times\mathbb R^d)$. A standard verification argument (by applying It\^o's formula to $e^{-\beta(s -t)}\bar J^1(s,X^{\pi^n}_s)$, integrating from $t$ to $T$ and then taking expectation)  yields
$J^1\equiv \bar J^1$. 
Now,  $\nabla  J^1\equiv \nabla \bar J^1 \in \mathcal{C}^{\frac{\alpha}{2},1+\alpha}([0,T)\times\mathbb R^d)
\cap \mathcal{C}([0,T]\times\mathbb R^d)$ and, hence,
$\log \pi^2(a\,|\,t,x)$ belongs to  $\mathcal{C}^{\frac{\alpha}{2}, \alpha}$ in $(t.x)$.
Repeating the above argument inductively yields $J^n\equiv \bar J^n$.
\end{proof}

\quad The following result of the BSDEs
will be used to bound $|J^n(t,x) - J^*(t,x)|$.
\begin{lemma}
\label{lem:BSDE}
Let $F_t(y,z,Z)$ be a measurable function of $(t,y,z,Z,\omega)$ such that $(F_t(y,z,Z), \, 0 \le t \le T)$ is predictable for any fixed $(y,z,Z)$ and
$(F_t(0,0,0), \, 0 \le t \le T) \in \mathbb{H}^0_0$,
and there exist $M_1, M_2, M_3 > 0$,
\begin{equation*}
|F_t(y',z',Z') - F_t(y,z,Z)| \le M_1 |y'-y| + M_2 |z'-z| + M_3 |Z' - Z| \quad a.s.
\end{equation*}
Then we have
\begin{enumerate}
\item[(i)]
For $\xi \in \mathbb{L}^2(\mathcal{F}_T)$, there exists a unique solution $(Y,Z)$ to
\begin{equation*}
Y_t = \xi + \int_t^T F_t(Y_s, Z_s, Z_s) ds - \int_t^T Z_s d\widetilde W_s, \quad 0 \le t \le T \quad a.s.
\end{equation*}
\item[(ii)]
For $\xi \in \mathbb{L}^2(\mathcal{F}_T)$ and $z \in \mathbb{H}^0_0$, there exists a unique solution $(Y,Z)$ to
\begin{equation}
\label{eq:iiYZ}
Y_t = \xi + \int_t^T F_t(Y_s, z_s, Z_s) ds - \int_t^T Z_s d\widetilde W_s, \quad 0 \le t \le T \quad a.s.
\end{equation}
\item[(iii)] Let $(Y^1, Z^1)$ and $(Y^2, Z^2)$ be the solutions to \eqref{eq:iiYZ} corresponding to $z^1\in \mathbb{H}^0_0$ and $z^2\in \mathbb{H}^0_0$ respectively.
Fixing $\eta \in (0,1)$, we have for $\theta \ge M_1 + (1+\eta^{-1}) (M_2+M_3)^2$,
\begin{equation*}
e^{\theta t} \mathbb{E}|Y^1_t - Y^2_t|^2  + |Z^1 - Z^2|^2_{\mathbb{H}^\theta_t} \le \eta |z^1 - z^2|^2_{\mathbb{H}^\theta_t}, \quad 0 \le t \le T.
\end{equation*}
\end{enumerate}
\end{lemma}
\begin{proof}
(i) and (ii) follow from the standard BSDE theory; see, e.g., \cite{Pham09}, Theorem 6.3.3,
and (iii) from the proof of \cite{KSS20}, Lemma A.5.
\end{proof}

\subsection{Proof of Theorem \ref{thm:cvrate}}

We need a few lemmas to prove Theorem \ref{thm:cvrate}.
Recall the definition of the optimal feedback control
 $\pi^*$ from \eqref{eq:pistargam},
 and let $X^{\pi^*}$ be the corresponding state process.
Denote (with a slight abuse of notation)
\begin{equation}
\label{eq:pipi}
\pi(a \,|\, s,x,z) \propto \exp \left( \frac{1}{\gamma} \left(b(s,x,a) \, \sigma^{-1}(s,x) z + r(s,x,a) \right)\right), 
\end{equation}
which is well-defined as $\sigma(s,x)$ is positive definite by Assumption \ref{assump}-(iii).
Then we can rewrite
\begin{equation*}
\pi^n(a \,|\, s,x) = \pi(a \,|\, s, x, \left(\sigma \, \nabla J^{n-1} \right)(s,x))
\quad \mbox{and} \quad
\pi^*(a \,|\, s,x) = \pi(a \,|\, s, x, \left(\sigma \, \nabla J^* \right)(s,x)).
\end{equation*}
Define
\begin{equation}
\label{eq:Gt}
\begin{aligned}
G_s(z,Z):= \int_\mathcal{A} \bigg[ b(s, X_s^{\pi^*}, a) \, \sigma^{-1} & (s, X_s^{\pi^*})Z + r(s,X_s^{\pi^*} ,a) \\
& - \gamma \log \pi(a \,|\, s,X_s^{\pi^*}, z) \bigg] \pi(a \,|\, s,X_s^{\pi^*}, z) \,da,
\end{aligned}
\end{equation}
which is the key to our analysis.
The next lemma studies the Lipschitz property of $G_s$.
\begin{lemma}
\label{lem:sensitivityG}
Let Assumption \ref{assump} hold.
Then for $|z|, |Z| \le C$, there exists $L > 0$ (independent of $\gamma$) such that
\begin{equation}
\label{eq:LipG}
|G_s(z', Z') - G_s(z,Z)| \le  L \left( 1+ e^{\frac{L}{\gamma}} + \frac{1}{\gamma} e^{\frac{L}{\gamma}} \right) |z'-z| + L|Z'-Z|
\quad a.s.,
\end{equation}
for all $z'$ and $Z'$.
\end{lemma}
\begin{proof}
Observe that
\begin{equation}
\label{eq:diffdecomp}
|G_s(z', Z') - G_s(z,Z)| \le E_1 + E_2 + E_3,
\end{equation}
where
\begin{align*}
& E_1:= \left| \int_\mathcal{A} r(s,X_s^{\pi^*} ,a) \left(\pi(a \,|\, s,X_s^{\pi^*}, z') - \pi(a \,|\, s,X_s^{\pi^*}, z) \right) da \right|, \\
& E_2:= \left| \int_\mathcal{A}  b(s, X_s^{\pi^*}, a) \, \sigma^{-1} (s, X_s^{\pi^*}) \left( Z' \pi(a \,|\, s,X_s^{\pi^*}, z') - Z \pi(a \,|\, s,X_s^{\pi^*}, z) \right) da \right|, \\
& E_3:= \gamma \left| \int_\mathcal{A} \left(\log \pi(a \,|\, s,X_s^{\pi^*}, z') \, \pi(a \,|\, s,X_s^{\pi^*}, z')  - \log \pi(a \,|\, s,X_s^{\pi^*}, z) \, \pi(a \,|\, s,X_s^{\pi^*}, z) \right) da \right|.
\end{align*}
By Assumption \ref{assump}-(iii), we get
\begin{equation}
\label{eq:E1TV}
E_1 \le \overline{r} \, \left| \pi(\cdot \,|\, s,X_s^{\pi^*}, z') - \pi(\cdot \,|\, s,X_s^{\pi^*}, z) \right|_{\mathbb{L}^1(\mathcal{A})}
= 2 \overline{r} \, d_{TV}\left(\pi(\cdot \,|\, s,X_s^{\pi^*}, z'), \, \pi(\cdot \,|\, s,X_s^{\pi^*}, z)\right).
\end{equation}
Recall that
\begin{equation*}
\pi(a \,|\, s,x,z) = \frac{\exp \left( \frac{1}{\gamma} (b(s,x,a) \, \sigma^{-1}(s,x)z + r(s,x,a) \right)}{\int_\mathcal{A}\exp \left( \frac{1}{\gamma} (b(s,x,a) \, \sigma^{-1}(s,x)z + r(s,x,a)\right)da}.
\end{equation*}
Applying \cite[Theorem 8]{Spr20},
\footnote{{For two Gibbs measures $\pi(da) = e^{\Psi(a)}da/Z$ and $\pi'(da) = e^{\Psi'(a)}da/Z'$,
\cite[Theorem 8]{Spr20} provides a bound on $d_{TV}(\pi(\cdot), \pi'(\cdot))$ under the assumption that $\sup \Psi = 0$.
This assumption can be  removed by modifying the bound to be
\begin{equation*}
d_{TV}(\pi(\cdot), \pi'(\cdot)) \le \frac{e^{\sup_a  (\Psi \vee \Psi')}}{Z} |\Psi - \Psi'|_{\mathbb{L}^1}.
\end{equation*}}} we get for $|z' - z| \le C$ (and hence $|z'| \le 2C$):
\begin{equation*}
\begin{aligned}
& d_{TV}\left(\pi(\cdot \,|\, s,x, z'), \, \pi(\cdot \,|\, s,x, z)\right) \\
& \qquad \le \frac{\exp \left(\frac{2}{\gamma}(\overline{b} \overline{\sigma} C + \overline{r})\right)}{\int_\mathcal{A}\exp \left( \frac{1}{\gamma} (b(s,x,a) \, \sigma^{-1}(s,x)z + r(s,x,a)\right)da} \left| \frac{1}{\gamma} b(s,x,a) \, \sigma^{-1}(t,x) (z'-z)\right|_{\mathbb{L}^1(\mathcal{A})}   \\
& \qquad \le \frac{\exp \left(\frac{3}{\gamma}(\overline{b} \overline{\sigma} C + \overline{r}\right)}{|\mathcal{A}|} \frac{\overline{b} \overline{\sigma} |z-z'| |\mathcal{A}|}{\gamma} = \frac{\overline{b} \overline{\sigma}}{\gamma} \exp \left(\frac{3}{\gamma}(\overline{b} \overline{\sigma} C + \overline{r})\right) |z'-z|.
\end{aligned}
\end{equation*}
Moreover, it is true that
$d_{TV}\left(\pi(\cdot \,|\, s,x, z'), \, \pi(\cdot \,|\, s,x, z)\right) \le 1$ (because the total variation distance of any two probability distributions is bounded by  $1$.)
So for $|z' - z| > C$, we have $d_{TV}\left(\pi(\cdot \,|\, s,x, z'), \, \pi(\cdot \,|\, s,x, z)\right) \le 1 < |z' - z|/C$.
Consequently,
\begin{equation}
\label{eq:TVdiff}
d_{TV}\left(\pi(\cdot \,|\, s,x, z'), \, \pi(\cdot \,|\, s,x, z)\right) \le \left( \frac{\overline{b} \overline{\sigma}}{\gamma} \exp \left(\frac{3}{\gamma}(\overline{b} \overline{\sigma} C + \overline{r})\right)+ \frac{1}{C}\right) |z'-z|.
\end{equation}
Combining \eqref{eq:E1TV} and \eqref{eq:TVdiff} yields
\begin{equation}
\label{eq:E1bd}
E_1 \le 2 \overline{r} \left( \frac{\overline{b} \overline{\sigma}}{\gamma} \exp \left(\frac{3}{\gamma}(\overline{b} \overline{\sigma} C + \overline{r})\right) + \frac{1}{C}\right) |z'-z|.
\end{equation}
Similarly, we get
\begin{equation}
\label{eq:E2bd}
\begin{aligned}
E_2 & \le \overline{b} \overline{\sigma} \left( \int_\mathcal{A} |Z'-Z| \, \pi(a \,|\, s,X_s^{\pi^*}, z') \, da + \int_\mathcal{A} |Z| \, |\pi(a \,|\, s,X_s^{\pi^*}, z') - \pi(a \,|\, s,X_s^{\pi^*}, z) | da\right) \\
& \le \overline{b} \overline{\sigma} \left(|Z'-Z| + 2 C \, d_{TV}\left(\pi(\cdot \,|\, s,X_s^{\pi^*}, z'), \, \pi(\cdot \,|\, s,X_s^{\pi^*}, z)\right) \right) \\
& \le \overline{b} \overline{\sigma} |Z'-Z| + \left(\frac{2\overline{b} \overline{\sigma} C}{\gamma} \exp \left(\frac{3}{\gamma}(\overline{b} \overline{\sigma} C + \overline{r}) \right) + 2\overline{b} \overline{\sigma}\right) |z'-z|.
\end{aligned}
\end{equation}
Finally, we have
\begin{equation}
\label{eq:E3bd0}
\begin{aligned}
E_3  &\le  \underbrace{\left| \int_\mathcal{A} r(s,X_s^{\pi^*} ,a) \left(\pi(a \,|\, s,X_s^{\pi^*}, z') - \pi(a \,|\, s,X_s^{\pi^*}, z) \right) da \right|}_{(a)} \\
& \qquad \quad +  \underbrace{\left| \int_\mathcal{A}  b(s, X_s^{\pi^*}, a) \, \sigma^{-1} (s, X_s^{\pi^*}) \left( z' \pi(a \,|\, s,X_s^{\pi^*}, z') - z \pi(a \,|\, s,X_s^{\pi^*}, z) \right) da \right|}_{(b)}  \\
& \qquad \quad + \underbrace{\gamma \left| \log \left( \frac{\int_\mathcal{A} \exp \left(\frac{1}{\gamma}(b(s, X_s^{\pi^*}, a) \, \sigma^{-1} (s, X_s^{\pi^*}) z' + r(s, X_s^{\pi^*},a))) \right)da }{\int_\mathcal{A} \exp \left(\frac{1}{\gamma}(b(s, X_s^{\pi^*}, a) \, \sigma^{-1} (s, X_s^{\pi^*}) z + r(s, X_s^{\pi^*},a))) \right)da} \right) \right|}_{(c)} \\
& \le \left(3 \overline{b} \overline{\sigma} +\frac{2 \overline{b} \overline{\sigma} (C+ \overline{r})}{\gamma} \exp \left(\frac{3}{\gamma}(\overline{b} \overline{\sigma} C + \overline{r}) \right)+ \frac{2 \overline{r}}{C} \right)|z'-z| + (c),
\end{aligned}
\end{equation}
where $(a)$ is bounded by $E_1$, and $(b)$ is bounded by $E_2$ with $(Z,Z') = (z,z')$.
Now we proceed to bounding the term $(c)$.
For $|z'| > 2C$, we have $|z' -z| > C$.
Thus,
\begin{equation}
\label{eq:cbd1}
\begin{aligned}
\left| \log \left( \frac{\int_\mathcal{A} \exp \left(\frac{1}{\gamma}(b(s, x, a) \, \sigma^{-1} (s, x) z' + r(s, x,a))\right)da }{\int_\mathcal{A} \exp \left(\frac{1}{\gamma}(b(s, x, a) \, \sigma^{-1} (s, x) z + r(s, x,a)) \right)da} \right) \right|
&  \le \log\left(\frac{e^{\frac{\overline{b} \overline{\sigma} |z'|+\overline{r}}{\gamma}}|A|}{e^{-\frac{\overline{b} \overline{\sigma} C+ \overline{r}}{\gamma}}|A|} \right) \\
& = \frac{\overline{b} \overline{\sigma} (|z'| + C) + 2 \overline{r}}{\gamma} \\
& \le \frac{\overline{b} \overline{\sigma} |z'-z| + 2 ( \overline{b} \overline{\sigma} C + \overline{r})  }{\gamma} \\
& \le \frac{3 \overline{b} \overline{\sigma} + 2 \overline{r}/C}{\gamma} |z'-z|,
\end{aligned}
\end{equation}
where the second inequality follows from $|z'| \le |z'-z| + |z| \le |z'-z| + C$,
and the last inequality is due to the fact that $|z'-z| > C$.

\quad Now consider the case when $|z'| \le 2C$.
Since $|\log(u'/u)| \le |u'-u|/(u \wedge u')$ for $u, u' \ge 0$, we get
\begin{equation}
\label{eq:cbd2}
\left| \log \left( \frac{\int_\mathcal{A} \exp \left(\frac{1}{\gamma}(b(s, x, a) \, \sigma^{-1} (s, x) z' + r(s, x,a))\right)da }{\int_\mathcal{A} \exp \left(\frac{1}{\gamma}(b(s, x, a) \, \sigma^{-1} (s, x) z + r(s, x,a)) \right)da} \right) \right|  \le \frac{(d)}{(e)},
\end{equation}
where
\begin{equation*}
\begin{aligned}
& (d):= \int_\mathcal{A} \bigg|\exp \left(\frac{1}{\gamma}(b(s, x, a) \, \sigma^{-1} (s, x) z' + r(s, x,a)) \right) \\
& \qquad  \qquad  \qquad  \qquad  \qquad  \qquad- \exp \left(\frac{1}{\gamma}(b(s, x, a) \, \sigma^{-1} (s, x) z + r(s, x,a)) \right)\bigg| da, \\
& (e):=\int_\mathcal{A} \exp \left(\frac{1}{\gamma}(b(s, x, a) \, \sigma^{-1} (s, x) z' + r(s, x,a)) \right)da \\
& \qquad  \qquad  \qquad  \qquad  \qquad  \qquad \wedge \int_\mathcal{A} \exp \left(\frac{1}{\gamma}(b(s, x, a) \, \sigma^{-1} (s, x) z + r(s, x,a)) \right)da.
\end{aligned}
\end{equation*}
By \cite{Spr20}, Theorem 5, we have
\begin{equation}
\label{eq:numdenest}
(d) \le  \frac{\overline{b} \overline{\sigma} | \mathcal{A}| }{\gamma} \exp \left(\frac{1}{\gamma}(\overline{b} \overline{\sigma} C + \overline{r})\right) |z'-z|,
\quad (e) \ge  |\mathcal{A}| \exp\left( - \frac{2 \overline{b} \overline{\sigma} C + \overline{r}}{\gamma} \right).
\end{equation}
The estimates \eqref{eq:cbd1}--\eqref{eq:numdenest} yield
\begin{equation}
\label{eq:cbd}
(c) \le \left(\overline{b} \overline{\sigma} \exp\left(\frac{3 \overline{b} \overline{\sigma} C + 2 \overline{r}}{\gamma} \right)  + 3\overline{b} \overline{\sigma} + \frac{2 \overline{r}}{C}  \right) |z' - z|.
\end{equation}
By \eqref{eq:E3bd0} and \eqref{eq:cbd}, we obtain
\begin{equation}
\label{eq:E3bd}
E_3 \le \left( 6 \overline{b} \overline{\sigma} + \frac{4 \overline{r}}{C} + \overline{b} \overline{\sigma} \left( 1 + \frac{2(C + \overline{r})}{\gamma}\right) \exp \left(\frac{3}{\gamma}(\overline{b} \overline{\sigma} C + \overline{r}) \right)  \right) |z'-z|.
\end{equation}
Combining \eqref{eq:diffdecomp} with \eqref{eq:E1bd}, \eqref{eq:E2bd} and \eqref{eq:E3bd} yields \eqref{eq:LipG}.
\end{proof}

\quad The following result establishes representations for $J^n$ and $J^{*}$.

\begin{lemma}
\label{lem:repJJ}
Let Assumption \ref{assump} hold.
There exist $(\widehat{W}, \widehat{\mathbb{P}})$, with $\widehat{W}$ being $\widehat{\mathbb{P}}$-Brownian motion, such that
\begin{enumerate}[itemsep = 3 pt]
\item[(i)]
We have
\begin{equation}
\label{eq:repJn}
J^n(t,x) = h(X^{\pi^*}_T) + \int_t^T [ G_s(Z^{n-1}_s, Z^n_s) - \beta Y^n_s] \, ds - \int_t^T Z^n_s d\widehat{W}_s,
\end{equation}
where $Y^n_s: = J^n(s, X_s^{\pi^*})$ and $Z^n_s:= \left(\sigma \, \nabla J^n \right)(s, X_s^{\pi^*})$.
\item[(ii)]
The BSDE
\begin{equation}
\label{eq:Ytxs}
Y^{t,x}_s = h(X^{\pi^*}_T)  + \int_s^T [G_u(Z^{t,x}_u, Z^{t,x}_u) - \beta Y^{t,x}_u ] \, du - \int_s^T Z^{t,x}_u d\widehat{W}_u, \quad t \le s \le T,
\end{equation}
has a unique solution $(Y^{t,x}, Z^{t,x})$ (the superscript stands for $X^{\pi^*}_t = x$).
Moreover, $J^*(t,x) = Y^{t,x}_t$.
\end{enumerate}
\end{lemma}
\begin{proof}
(i) Applying It\^o's formula to $J^n (s, X_s^{\pi^*})$ and noting Lemma \ref{lem:PEFK},
we get
\begin{equation}
\label{eq:chara}
\begin{aligned}
& dJ^n (s, X_s^{\pi^*}) \\
& = \left[ \partial_t J^n (s, X_s^{\pi^*}) + \frac{1}{2} \tr(\sigma \sigma^T \nabla^2 J^n) (s, X_s^{\pi^*})
+ \widetilde{b}(s, X_s^{\pi^*}, \pi^*) \, \nabla J^n (s, X_s^{\pi^*})\right]ds  \\
& \qquad  \qquad  \qquad   \qquad  \qquad  \qquad \qquad  \qquad  \qquad \qquad  \qquad \qquad \qquad \quad+ \left( \sigma \, \nabla J^n \right) (s, X_s^{\pi^*}) d\widetilde W_s \\
& = \bigg[  \widetilde{b}(s, X_s^{\pi^*}, \pi^*) \, \nabla J^n (s, X_s^{\pi^*}) -\widetilde{b}(s, X_s^{\pi^*}, \pi^n) \, \nabla J^n (s, X_s^{\pi^*}) - \widetilde{r} (s, X_s^{\pi^*}, \pi^n)  \\
& \qquad  + \gamma \int_\mathcal{A} \log \pi^n (a \,|\, s, X_s^{\pi^*}) \, \pi^n (a \,|\, s, X_s^{\pi^*}) da + \beta J^n (s, X_s^{\pi^*}) \bigg] ds + \left( \sigma \, \nabla J^n \right)(s, X_s^{\pi^*}) d\widetilde W_s,
\end{aligned}
\end{equation}
where $\pi^*$ and $\pi^n$ appearing in $\widetilde{b}\left(s, X_s^{\pi^*}, \cdot \right)$, $\widetilde{r}\left(s, X_s^{\pi^*}, \cdot \right)$ stand for $\pi^* (\cdot \,|\, s, X_s^{\pi^*})$ and $\pi^n (\cdot \,|\, s, X_s^{\pi^*})$ respectively.

\quad Let $\widehat{\mathbb{P}}$ be a change of measure of $\mathbb{P}$,
with the Radon--Nikodym derivative:
\begin{equation*}
\frac{d\widehat{\mathbb{P}}}{d \mathbb{P}} = \mathcal{E}\left( -\int_0^{\cdot}  \widetilde{b}(s, X^{\pi^*}_s, \pi^*) \, \sigma^{-1}(s, X^{\pi^*}_s) \, d\widetilde W_s\right)_T,
\end{equation*}
where $\mathcal{E}(\cdot)$ is the stochastic exponential (see \cite{RY99}, p.328).
So $\widehat{W}_t: = \widetilde W_t + \int_0^t  \widetilde{b}(s, X^{\pi^*}_s, \pi^*) \, \sigma^{-1}(s, X^{\pi^*}_s) \, ds$
is a $\widehat{\mathbb{P}}$-Brownian motion.

\quad Set
$Y^n_s: = J^n(s, X_s^{\pi^*})$ and $Z^n_s:= \left(\sigma \, \nabla J^n \right)(s, X_s^{\pi^*})$;
so $Y^n_t = J^n(t,x)$.
By integrating \eqref{eq:chara} on $[t,T]$ and noting that $\pi^n (a \,|\, s, X_s^{\pi^*}) = \pi(a \,|\, s, X_s^{\pi^*}, Z^{n-1}_s)$, we get
\begin{equation*}
Y^n_t = g(X^{\pi^*}_T) - \int_t^T \bigg[ \widetilde{b}(s, X_s^{\pi^*}, \pi^*) \,  \sigma^{-1}(s, X^{\pi^*}_s) \, Z^n_s
- G_s(Z^{n-1}_s, Z^n_s) + \beta Y^n_s \bigg] ds - \int_t^T Z^n_s d\widetilde W_s,
\end{equation*}
which leads to \eqref{eq:repJn}.

(ii) The fact that the BSDE \eqref{eq:Ytxs} has a unique solution follows from Lemma \ref{lem:BSDE}-(i) noting Lemma \ref{lem:sensitivityG}.
We have then
\begin{equation*}
\begin{aligned}
dY^{t,x}_s &= \bigg(\widetilde{b}(s, X^{\pi^*}_s, \pi^*) \sigma^{-1}(s, X^{\pi^*}_s)Z^{t,x}_s + \beta Y^{t,x}_s  \\
& \qquad - \gamma \log \int_\mathcal{A} \exp \left( \frac{1}{\gamma} (b(u, X^{\pi^*}_u, a) \sigma^{-1}(s, X^{\pi^*}_s)Z^{t,x}_s  + r(u, X^{\pi^*}_u, a) ) \right) da \bigg) ds + Z_s^{t,x} d\widetilde W_s.
\end{aligned}
\end{equation*}
A standard BSDE argument (see \cite{EPQ97}, Proposition 4.3 or \cite{Pham09}, Section 6.3) shows that $Y^{t,x}_t = v(t,x)$ is the solution to the semi-linear PDE:\footnote{The semi-linear PDE can be written as
$\partial_t v + \frac{1}{2} \tr\left(\sigma \sigma^T \nabla^2 v \right) + f(t,x,v,p) = 0$,
where
\begin{equation*}
f(t,x,v,p):= \gamma \log \int_{\mathcal{A}} \exp \left(\frac{1}{\gamma} (b(t,x,a) p + r(t,x,a)) \right)da - \beta v.
\end{equation*}
It follows from  Assumption \ref{assump} that $f$ is $\mathcal{C}^{\frac{\alpha}{2},\alpha}$ in $(t,x)$ and Lipschitz in $(v,p)$.
Classical Schauder's theory (see e.g., \citealp[Chapter 7]{Fri64}) implies that the PDE has a unique solution $\bar v \in  \mathcal{C}^{1+\frac{\alpha}{2},2+\alpha}([0,T)\times\mathbb R^d)
\cap \mathcal{C}([0,T]\times\mathbb R^d)$, which then coincides with $v$ by a standard verification argument.}
\begin{equation*}
\left\{ \begin{array}{lcl}
\partial_t v + \frac{1}{2} \tr\left(\sigma \sigma^T \nabla^2 v \right) + \gamma \log \int_{\mathcal{A}} \exp \left(\frac{1}{\gamma} (b(t,x,a) \nabla v + r(t,x,a)) \right)da - \beta v = 0, \\
v(T,x) = h(x).
\end{array}\right.
\end{equation*}
Note that $$\gamma \int_\mathcal{A} H(t,x,a, \nabla v, \nabla^2 v) \, da = \frac{1}{2} \tr\left(\sigma \sigma^T \nabla^2 v \right) + \gamma \log \int_{\mathcal{A}} \exp \left(\frac{1}{\gamma} (b(t,x,a) \nabla v + r(t,x,a)) \right)da.$$
Thus, $Y_t^{t,x}$ is the solution to the HJB equation \eqref{eq:expHJB}, i.e. $Y_t^{t,x} = J^*(t,x)$.
\end{proof}

\quad Now we prove Theorem \ref{thm:cvrate}.
\begin{proof} (of Theorem \ref{thm:cvrate})
Let $F_s(Y,z,Z): = G_s(z,Z) - \beta Y$.
It follows from Lemma \ref{lem:repJJ} that
$J^n(t,x) = Y^n_t$ and $J^*(t,x) = Y^{t,x}_t$ where
\begin{equation*}
\begin{aligned}
& Y^n_s = h(X^{\pi^*}_T) + \int_s^T F_u(Y^n_u, Z^{n-1}_u, Z^n_u) \, du - \int_s^T Z^n_u d\widehat{W}_u, \\
& Y^{t,x}_s = h(X^{\pi^*}_T)  + \int_s^T F_u(Y^{t,x}_u, Z^{t,x}_u, Z^{t,x}_u) \, du - \int_s^T Z^{t,x}_u d\widehat{W}_u.
\end{aligned}
\end{equation*}
Note that $|\nabla J^*| \le M$ (with $M$ independent of $\gamma$; see \citealp{TZZ21}), and $Z^{t,x}_s= \left(\sigma \nabla J^* \right)(s, X^{\pi^*}_s)$.
By Lemma \ref{lem:sensitivityG}, we get
\begin{equation*}
\begin{aligned}
& |F_u(Y^n_u, Z^{n-1}_u, Z^n_u) - F_u(Y^{t,x}_u, Z^{t,x}_u, Z^{t,x}_u)| \\
& \qquad \le \beta |Y^n_u - Y^{t,x}_u| +  L \left( 1+ e^{\frac{L}{\gamma}} + \frac{1}{\gamma} e^{\frac{L}{\gamma}} \right) |Z^{t,x}_u- Z^{n-1}_u| + L|Z^{t,x}_u-Z^n_u|.
\end{aligned}
\end{equation*}
Applying Lemma \ref{lem:BSDE}-(iii) with $M_1 = \beta$, $M_2 = L \left( 1+ e^{\frac{L}{\gamma}} + \frac{1}{\gamma} e^{\frac{L}{\gamma}} \right)$ and $M_3 = L$, we have for $\theta \ge \beta + (1+\eta^{-1}) L^2\left( 1+ e^{\frac{L}{\gamma}} + \frac{1}{\gamma} e^{\frac{L}{\gamma}} \right)^2$,
\begin{equation*}
e^{\theta t} \widehat{\mathbb{E}}|Y^{t,x}_t - Y^n_t|^2  + |Z^{t,x} - Z^n|^2_{\widehat{\mathbb{H}}^\theta_t} \le \eta |Z^{t,x} - Z^{n-1}|^2_{\widehat{\mathbb{H}}^\theta_t},
\end{equation*}
where $\widehat{\mathbb{H}}^\theta_t$ is defined under $\widehat{\mathbb{P}}$.
So we get
$|Z^{t,x} - Z^n|^2_{\widehat{\mathbb{H}}^\theta_t} \le \eta^n |Z^{t,x} - Z^0|^2_{\widehat{\mathbb{H}}^\theta_t}$, and
\begin{equation*}
\begin{aligned}
|J^*(t,x) - J^n(t,x)|^2  & \le \eta^n \widehat{\mathbb{E}}\int_t^T e^{\theta (s-t)} |\sigma(s, X^{\pi^*}_s)|^2 | \nabla J^*(s, X^{\pi^*}_s) - \nabla J^0(s, X^{\pi^*}_s)|^2 ds \\
& \le \eta^n e^{\theta (T-t)} \widehat{\mathbb{E}}\int_t^T |\sigma(s, X^{\pi^*}_s)|^2 | \nabla J^*(s, X^{\pi^*}_s) - \nabla J^0(s, X^{\pi^*}_s)|^2 ds
\end{aligned}
\end{equation*}
which proves \eqref{eq:expcv}.
\end{proof}

\section{$q$-learning and regret}
\label{sc4}

\quad In this section, we study the $q$-learning algorithm \eqref{eq:SGD0}--\eqref{eq:sample0}
and derive  its regret.
As an intermediate step,
we consider a ``semi"-$q$-learning algorithm in Section \ref{sc41},
where value functions  of given policies   can be accessed
and we only need to learn the $q$-function. The advantage of studying the semi-$q$-learning first is to present the essential idea while simplifying notation.
We will show how to adapt the arguments in Section \ref{sc3}
to derive the convergence rate of the semi-$q$-learning algorithm.
The result for the $q$-learning algorithm is then provided in Section \ref{sc42}.

\subsection{Semi-$q$-learning}
\label{sc41}

Here we assume an oracle access to the value functions $J^n$,
and the resulting algorithm is called the semi-$q$-learning.
We start with some initial parameter value $\phi^1$ for the $q$-function
and
a control policy $\pi^{1}(\cdot \,|\, \cdot,\cdot)$,
and for $n \ge 1$,
\begin{enumerate}[itemsep = 3 pt]
\item
Update
\begin{equation}
\label{eq:SGD1}
\phi_{n+1} = \phi_n + \alpha_{\phi,n} \int_0^T \int_t^T e^{-\beta(s-t)} \frac{\partial q^\phi}{\partial \phi}_{|\phi = \phi_n}(s, X_s^{\pi^{n}}, a_s^{\pi^{n}}) ds G^n_{t:T}dt,
\end{equation}
where
$$G^n_{t:T}: = e^{-\beta (T-t)} h(X_T^{\pi^{n}}) - J^n(t,X_t^{\pi^{n}})
+ \int_t^T e^{-\beta(s-t)} [r(s, X^{\pi^{n}}_s, a^{\pi^{n}}_s) - q^{\phi_n}(s, X^{\pi^{n}}_s, a^{\pi^{n}}_s)] ds.$$
\item
Sample
\begin{equation}
\label{eq:sample10}
\pi^{n+1}(a \,|\, t,x) \propto \exp\left(\frac{1}{\gamma} q^{\phi_{n+1}}(t,x,a) \right)da.
\end{equation}
\end{enumerate}

\quad Recall \eqref{eq:Jngam} that defines $J^n$
and it follows
from \eqref{eq:qfunc} that
\begin{equation*}
q(t,x,a; \pi^{n}) = b(t,x,a) \nabla J^n(t,x) + r(t,x,a) + \frac{\partial J^n}{\partial t}(t,x)  + \frac{1}{2} \tr\left(\sigma^2(t,x) \nabla^2 J^n(t,x) \right) -  \beta J^n(t,x).
\end{equation*}
To avoid undue technicality, we assume throughout this subsection that $\nabla J^n \ne 0$ almost everywhere.\footnote{It is known (\citealp{WZ20}) that for linear--quadratic (LQ) problems, $\nabla J^n(t,x) = A_n(t)  x + B_n(t)$ with $A_n(t) > 0$ for a suitably chosen initial policy. This satisfies $\nabla J^n(t,x) \ne 0$ almost everywhere. The example presented immediately after Corollary \ref{coro:45} is another instance satisfying this assumption.
In these examples, $b^n(t,x,a)$ is well-defined.
Later, an additional assumption on the uniform boundedness of $b^n$
 (Assumption~\ref{assump2}) is required in order to establish the convergence of the learning algorithm.
 We note, however, that LQ problems fall outside this bounded-coefficient setting.}
Let $b^{n+1}(t,x,a)$ be a measurable function
such that the following holds almost everywhere\footnote{Because $\nabla J^n \ne 0$ almost everywhere,  $b^n(t,x,a)$ is uniquely determined if $d = 1$, and are solutions to the underdetermined system \eqref{eq:bn} if $d \ge 2$.
The measurability follows from standard measurable selection theory, e.g., \cite{Wagner77}.}:
\begin{equation}
\label{eq:bn}
\begin{aligned}
& q^{\phi_{n+1}}(t,x,a) \\
& = b^{n+1}(t,x,a) \nabla J^n(t,x) + r(t,x,a) + \frac{\partial J^n}{\partial t}(t,x)  + \frac{1}{2} \tr\left(\sigma^2(t,x) \nabla^2 J^n(t,x) \right) -  \beta J^n(t,x).
\end{aligned}
\end{equation}
The policy update \eqref{eq:sample10} can now be written as:
\begin{equation}
\tag{41'}
\label{eq:sample20}
\pi^{n+1}(a \,|\, t,x) \propto \exp\left(\frac{1}{\gamma} (b^{n+1}(t,x,a) \nabla J^n(t,x)+ r(t,x,a)) \right)da.
\end{equation}
As $q^{\phi_{n+1}}(t,x,a)$ is expected to
be close to $q(t,x,a; \pi^n)$ when $n$ is large,
so is $b^n(t,x,a)$ to $b(t,x,a)$.
To present our result, we need the following condition in addition to Assumption \ref{assump}.
\begin{assumption}
\label{assump2}
There exists $\overline{b} > 0$ such that $\ess_{t,x,a} |b^n| \le \overline{b}$ for all $n$.
\end{assumption}

\quad The result below shows how the convergence of the value functions
depends in an explicit way on that of the $q$ values.
\begin{theorem}
\label{thm:main2}
Let Assumptions \ref{assump} and \ref{assump2} hold,
and assume that there is $M > 0$ such that
$|\nabla J^n| \le M$ for all $n$.
Fix $\eta \in (0,1)$. There exist $L$, $C> 0$ (independent of $\gamma$ and $n$) such that
\begin{equation}
\label{eq:cv2}
|J^*(t,x) - J^n(t,x)|^2 \le C \left( \eta^n e^{\Lambda(\gamma) (T -t)} + \left(1 + \frac{1}{\gamma} e^{\frac{L}{\gamma}}\right) \sum_{k = 1}^n \eta^{n-k} |q(\cdot;\pi^k) - q^{\phi_{k+1}}|_\infty \right),
\end{equation}
where $\Lambda(\gamma):= \beta + (1+\eta^{-1}) L^2\left[ 1+ e^{\frac{L}{\gamma}} + \frac{1}{\gamma} e^{\frac{L}{\gamma}} +
L \left( 1 + \frac{1}{\gamma} e^{\frac{L}{\gamma}}\right) \sup_n |q(\cdot;\pi^n) - q^{\phi_{n+1}}|_\infty\right]^2$.
\end{theorem}

\quad A proof of Theorem \ref{thm:main2} is delayed to Section \ref{sc43},
which is a ``perturbed" variant of the arguments in Section \ref{sc3}.
Here let us make several comments.
First,
the assumption $|\nabla J^n| \le M$ for all $n$ is again to avoid excessive technicality.
Because  $\nabla J^*$ is bounded and $J^n$ is expected to be a proxy to $J^*$,
it is reasonable to assume that $\nabla J^n$ is uniformly bounded.
Indeed, this assumption 
follows from a smallness condition on the model parameters.
By a heat kernel estimate (or Bismut's formula, \citealp{EL94}),
we get
$|\nabla J^n|_\infty \le C_1(1 + |f_n|_\infty)$,
where $f_n(t,x):=\int_\mathcal{A} \left( r(t,x,a) - \gamma \log \pi^n(a\,|\,t,x)\right) \pi^n(a\,|\,t,x) da$.
Here the constant $C_1$ depends on the uniform ellipticity and regularity of $\sigma$,
the time horizon $T$, the uniform bound $\bar{b}$ of $\int_\mathcal{A} b(t,x,a) \pi^n(a\,|\,t,x) da$,
and the uniform bound $K$ of $|\nabla h|_\infty$, but not on $n$.
Next, recall that $\pi^n(a\,|\,t,x) \propto \exp\left( \frac{b(t,x,a) \nabla J^{n-1}(t,x) + r(t,x,a) }{\gamma}\right)$.
We have
$|f_n|_\infty \le C_2(1 + |\nabla J^{n-1}|_\infty)$,
where $C_2$ depends on the boundedness of the model parameters -- it is again independent of $n$.
As a result,
\begin{equation*}
|\nabla J^n|_\infty \le C_1(1+C_2) + C_1C_2 |\nabla J^{n-1}|_\infty.
\end{equation*}
If the smallness condition $C_1C_2 < 1$ holds,
the above recursion implies that $\sup_n |\nabla J^n| < \infty$.

\quad Second, the theorem quantifies the performance of the algorithm
by how well the $q$ values are learned.
We illustrate this with two cases:
\begin{itemize}[itemsep = 3 pt]
\item
If $\inf_n |q(\cdot;\pi^n) - q^{\phi_{n+1}}|_\infty:= \delta > 0$,
i.e. the $q$ values are not well approximated,
then the right side of \eqref{eq:cv2} is bounded from below by
$ \sqrt{1-\eta^n}\geq 1-\eta^n$,
yielding at least a linear regret.
\item
If the $q$ values are well learned, say $|q(\cdot;\pi^n) - q^{\phi_{n+1}}|_\infty \approx n^{-\alpha}$ for some $\alpha > 0$,
then \eqref{eq:cv2} specializes to:
\begin{equation*}
|J^*(t,x) - J^n(t,x)|^2  \lesssim \sum_{k = 1}^n \eta^{n-k} k^{-\alpha} \lesssim n^{-\alpha} \ln n,
\end{equation*}
where the second inequality is obtained by
splitting the sum into $[1, n - \ln n]$ and $[n - \ln n, n]$
and by taking $\eta$ to be sufficiently small.
The regret now is sublinear:
\begin{equation}
\label{eq:multbound}
\sum_{k = 1}^n |J^*(t,x) - J^k(t,x)| \lesssim
\left\{ \begin{array}{rcl}
n^{1 - \frac{\alpha}{2}} (\ln n)^{\frac{1}{2}} & \mbox{if}
& \alpha <2, \\
(\ln n)^{\frac{3}{2}} & \mbox{if} & \alpha = 2, \\
\mathcal{O}(1) & \mbox{if} & \alpha > 2.
\end{array}\right.
\end{equation}
\end{itemize}

\quad Next we consider the convergence of the $q$ values via the iteration \eqref{eq:SGD1}.
This is an instance of the Robbins--Monro algorithm,
which can be written as:
\begin{equation}
\label{eq:SGD3}
\phi_{n+1} = \phi_n + \alpha_{\phi, n} \mathcal{H}(\phi_n, \underbrace{(X^{\pi^n}_t), (a^{\pi^n}_t)}_{U_{n+1}}),
\end{equation}
where
\begin{equation}
\label{eq;Hbig}
\mathcal{H}(\phi, U):= \int_0^T \int_t^T e^{-\beta(s-t)} \frac{\partial q^\phi}{\partial \phi}(s, X_s^{\pi^{\phi}}, a_s^{\pi^{\phi}}) ds G^\phi_{t:T}dt,
\end{equation}
with $U = ((X_t^{\pi^\phi}), (a_t^{\pi^\phi}))$,
$\pi^\phi(a \,|\, t,x)\propto \exp\left(\frac{1}{\gamma} q^\phi(t,x,a)\right)$
and $G^\phi_{t:T}: = e^{-\beta (T-t)} h(X_T^{\pi^{\phi}}) - J(t,X_t^{\pi^{\phi}}; \pi^\phi)
+ \int_t^T e^{-\beta(s-t)} [r(s, X^{\pi^{\phi}}_s, a^{\pi^{\phi}}_s) - q^{\phi}(s, X^{\pi^{\phi}}_s, a^{\pi^{\phi}}_s)] ds$.\footnote{\label{bias}
As commented  in Subsection \ref{qlearning} the action process $a^{\pi^\phi}$ is  discretely  sampled  to ensure  the required measurability.
Such defined $a^{\pi^\phi}$ will incur a discretization error of order $\max_{0 \le k \le N-1} |t_{k+1} - t_k|$ for $\phi$,
which vanishes as the time step size decreases.
For ease of presentation, we omit to write out the effect of discretization in the results below.
If time is equipartitioned,
it will add a term $N^{-\frac{\rho_\phi}{2}}$  (resp. $nN^{-\frac{\rho_\phi}{2}}$)
in \eqref{eq:imp} (resp. \eqref{eq:summ}).}
Clearly,
$$(U_{n+1} \,|\, U_n, \ldots, U_1, \phi_n, \ldots, \phi_1) \stackrel{d}{=} (U_{n+1} \,|\, \phi_n).$$
Also define
\begin{equation}
l(\phi):= \mathbb{E}\mathcal{H}(\phi, U).
\end{equation}

\quad We need a few more assumptions to study the convergence rate of
the stochastic approximation \eqref{eq:SGD3} (or \eqref{eq:SGD1}).
\begin{assumption}
\label{assump3}
~
\begin{enumerate}[itemsep = 3 pt]
\item
The ODE $\phi'(t) = l(\phi(t))$ has a unique stable equilibrium $\phi^*$.
\item
There exists $C > 0$ such that
$\mathbb{E}(|\mathcal{H}(\phi_n, U_{n+1})|^2\,|\, \phi_n) \le C(1 + |\phi_n|^2)$.
\item
There exists $\kappa > 0$ such that
$(\phi - \phi^*)l(\phi) \le - \kappa |\phi - \phi^*|^2$ in a neighborhood of $\phi^*$.
\item
There exist $\rho_\phi, C > 0$ such that
$|q^\phi - q^{\phi'}|_\infty \le C |\phi - \phi'|^{\rho_\phi}$ for all $\phi, \phi'$.
\item
There exists $C > 0$ such that
$|q(\cdot; \pi^\phi) - q(\cdot; \pi^*)|_\infty \le C d_{TV}(\pi^\phi, \pi^*)$ for $\phi$ in a neighborhood of $\phi^*$.
\end{enumerate}
\end{assumption}

\quad Set $\Delta:= |q(\cdot; \pi^*) - q^{\phi^*}(\cdot)|_\infty$. The value of $\Delta$ quantifies how close the family of functions $\{q^\phi\}_{\phi}$ approximates the optimal $q$ function.
When $\Delta = 0$ (if rarely the case), the family $\{q^\phi\}_{\phi}$ is rich enough to contain the optimal $q$ function.
The condition (1) ensures the optimal $\phi^*$ is the only candidate for the stochastic approximation.
The conditions (2)--(5) will be used to quantify the convergence rate of the $q$ values,
where (2)--(3) are growth conditions of the SGD \eqref{eq:SGD3},
(4) imposes H\"older regularity of the function approximation $\{q^\phi\}_{\phi}$,
and  (5) specifies the sensitivity of the $q$ function with respect to the (stochastic) policies.
Later we will give an example in which these conditions are satisfied.

\quad The following result provides an error estimate of
the semi-$q$-learning \eqref{eq:SGD1}--\eqref{eq:sample10}.
\begin{theorem}
\label{thm:main3}
Let the assumptions in Theorem \ref{thm:main2} and Assumption \ref{assump3} hold.
Set $\alpha_{\phi,n} = \frac{A}{n^{\nu} + B}$ for some $\nu \le 1$, $A > \frac{\alpha}{2 \kappa}$ and $B > 0$,
and let $\varepsilon > 0$.
Then there exists $C > 0$ (independent of $n, \varepsilon$) such that
with probability $1 - \varepsilon$,
\begin{equation}
\label{eq:imp}
|J^*(t,x) - J^n(t,x)| \le C \left(\Delta + \frac{1}{\varepsilon^{\rho_\phi/2}} n^{-\frac{\nu \rho_\phi}{4}} (\ln n)^{\frac{1}{2}} \right),
\end{equation}
for $\Delta$ sufficiently small.
Consequently, the regret is:
\begin{equation}
\label{eq:summ}
\sum_{k = 1}^n |J^*(t,x) - J^k(t,x)| \lesssim \Delta n + \frac{1}{\varepsilon^{\rho_\phi/2}}
\left\{ \begin{array}{rcl}
n^{1 - \frac{\nu \rho_\phi}{4}} (\ln n)^{\frac{1}{2}} & \mbox{if}
& \nu \rho_\phi <4, \\
(\ln n)^{\frac{3}{2}} & \mbox{if} & \nu \rho_\phi = 4, \\
\mathcal{O}(1) & \mbox{if} & \nu \rho_\phi > 4.
\end{array}\right.
\end{equation}
\end{theorem}

\quad Note that in the reinforcement learning and Markov decision processes literature,
the regret is typically defined cumulatively, i.e.,
the total reward loss incurred by not following the optimal policy.
Accordingly, in our entropy-regularized continuous-time setting,
it is defined by
$\sum_{k=1}^n (J^*-J^k)$,
which is bounded from above by $\sum_{k=1}^n |J^*-J^k|$.
 A proof of Theorem \ref{thm:main3} will be given in Section \ref{sc43}.
As previously mentioned, the error term $\Delta$ in \eqref{eq:imp} (or \eqref{eq:summ}) comes from function approximations, which is typically related to the quality of neural nets used, something that is not dictated or controlled by the $q$-learning in general.
If the family of functions $\{q^\phi\}_\phi$ include the optimal $q$ function,
then we get the following corollary.

\begin{corollary}
\label{coro:45}
Under the setting of Theorem \ref{thm:main3} with $\Delta = 0$,
there exists $C > 0$ (depending on $\gamma$ but not on $n, \varepsilon$) such that
with probability $1 - \varepsilon$,
\begin{equation}
|J^*(t,x) - J^n(t,x)| \le  \frac{C}{\varepsilon^{\rho_\phi/2}} n^{-\frac{\nu \rho_\phi}{4}} (\ln n)^{\frac{1}{2}}.
\end{equation}
Consequently, 
the regret $\sum_{k = 1}^n |J^*(t,x) - J^k(t,x)|$ is given by the second term in \eqref{eq:summ}.
\end{corollary}
\quad In particular, taking $\nu = 1$ and $\rho_\phi = 1$ (Lipschitz),
we get a {\it sublinear} $n^{\frac{3}{4}}$-regret bound.

\quad Now we illustrate Assumption \ref{assump3} with an example.
\begin{example} \label{ex1}
\em Set 
\begin{equation*}
b(t,x,a) = \alpha x + \rho a, \quad \sigma(t,x,a) = \sigma, \quad
r(t,x,a) = rx, \quad  h(x) = \ell x, \quad \mbox{and } \beta = 0
\end{equation*}
with the action space $\mathcal{A} = [0,1]$.
As mentioned in Footnote 1,
the dynamics  parameters $\alpha, \rho$ and $\sigma$ are unknown,
while the reward parameters $r$ and $\ell$ may be specified by the users.
Given a policy $\pi$, denote by
$E(\pi):= \int_{\mathcal{A}}a \pi(a) da$ and $\mbox{Ent}(\pi):=-\int_{\mathcal{A}} \pi(a) \log \pi(a)da$
its mean and differential entropy respectively.
We discuss two cases based on how policies are  parameterized.

\medskip
\noindent
{\em Case 1 (three-parameter policies)}. Assume $\alpha \ne 0$.
In this case
\begin{equation*}
J(t,x; \pi) = A(t) x + B^{\pi}(t) \quad \mbox{and} \quad
q(t,x,a; \pi) = \rho A(t) \left(a - E(\pi_t) \right) - \gamma \mbox{Ent}(\pi_t),
\end{equation*}
where $A(t):= e^{\alpha(T-t)}\left(\ell + \frac{r}{\alpha} \right) - \frac{r}{\alpha}$
and $B^\pi(t):= \int_t^T [\rho A(s) E(\pi_s) + \gamma \mbox{Ent}(\pi_s)] ds$.
It is clear that $\nabla J(t,x; \pi) = A(t) \ne 0$ almost everywhere.
The expressions of $q(t,x,a;\pi)$ and $A(t)$ suggest to parameterize:
\begin{equation*}
q_\phi(t,x,a) = F_\phi(t)(a -E(\pi^\phi_t))  - \gamma \mbox{Ent}(\pi^\phi_t),
\quad
\mbox{with } F_\phi(t):= \phi_1 e^{\phi_2(T-t)} + \phi_3,
\end{equation*}
and $\pi^\phi_t(a) \propto e^{-F_\phi(t)a/\gamma}$ on $\mathcal{A} = [0,1]$.
Define $w_i(t):= \partial_{\phi_i}F_\phi(t)$,
and $l_i(\phi)$ to be the $i^{th}$ coordinate of $l(\phi)$ for $i  \in \{1,2,3\}$.
We have:
\begin{equation*}
w_1(t) = e^{\phi_2(T-t)}, \quad w_2(t) = \phi_1(T-t)e^{\phi_2(T-t)}, \quad w_3(t)=1.
\end{equation*}
A direct computation yields:
\begin{equation*}
l_i(\phi) = \int_0^T \int_0^T  (s \wedge u) (\rho A(s) - F_\phi(s)) E(\pi^\phi_s) w_i(u) E(\pi^\phi_u)dsdu.
\end{equation*}

\quad Let us now check Assumption \ref{assump3}.
Clearly, $l(\phi^*)=0$ with  $\phi^*\equiv (\phi_1^*, \phi_2^*, \phi_3^*)= \left(\rho\left(\ell + \frac{r}{\alpha}\right),  \alpha, -\frac{\rho r}{\alpha}\right)$.
Denote  $w_i^*(t) = w_i(t)|_{\phi = \phi^*}$.
Assume further that
\begin{equation*}
\phi_1^* = \rho\left(\ell + \frac{r}{\alpha}\right) \ne 0,
\end{equation*}
and restrict the policy parameters to a compact set
$\mathcal{K} = \{\phi: |\phi_i| < C, \, i=1,2,3,\; |\phi_1| \ge c , \, \phi_2 \in I\}$ with $I$ a compact interval containing $\alpha$ but not $0$,
for some $c, C > 0$.
In this case, $F_\phi$ is injective on $\mathcal{K}$.
Compute the Jacobian matrix
$D l(\phi^*) \equiv  -G$, where
\begin{equation*}
G_{ij}:= \int_0^T \int_0^T (s \wedge u)w_j^*(s)E(\pi^{\phi^*}_s)w_i^*(u)E(\pi^{\phi^*}_u)dsdu.
\end{equation*}
Since $\phi_1^*, \phi_2^* \ne 0$, the functions $(w_1^*, w_2^*, w_3^*)$ are linearly independent.
Combining with the fact $E(\pi^{\phi^*}_t) > 0$ implies that $G$ is positive definite.
As a result, $l$ has a unique stable equilibrium $\phi^*$ and is locally dissipative,
i.e., the conditions (1) and (3) of Assumption \ref{assump3} are satisfied. 
Moreover, the condition (2) holds automatically  because
$\mathcal{A}$ is compact and $X^{\pi^\phi}$ is Gaussian.
Next, 
since $q_\phi$ is smooth in $\phi\in \mathcal{K}$, it is Lipschitz leading to the condition (4) with $\rho_\phi = 1$.
Finally,
$|q(\cdot; \pi^\phi) - q(\cdot; \pi^*)| \le C |E(\pi^\phi) - E(\pi^*)| + |\mbox{Ent}(\pi^\phi) - \mbox{Ent}(\pi^*)|
 \le 2\left(C + \gamma \max \log \pi^* + \gamma \max \frac{1}{c \wedge \pi^*} \right) d_{TV}(\pi^\phi, \pi^*)$
for some constants $C > 0$, and $c > 0$ (depending on the neighborhood of $\phi^*$).
This yields the condition (5).

\medskip
\noindent
{\em Case 2 (one-parameter policies)}. Assume $\alpha = 0$ and $\rho \ne 0$.
We have:
\begin{equation*}
J(t,x; \pi) = A(t) x + B^{\pi}(t) \quad \mbox{and} \quad
q(t,x,a; \pi) = \rho A(t) \left(a - E(\pi_t) \right) - \gamma \mbox{Ent}(\pi_t),
\end{equation*}
where $A(t):=\ell + r(T-t)$ and $B^\pi(t):= \int_t^T [\rho A(s) E(\pi_s) + \gamma \mbox{Ent}(\pi_s)] ds$.
Note that $\nabla J(t,x; \pi) \ne 0$ almost everywhere, provided that $r \ne 0$ or $\ell\neq0$. 
The main difference from Case 1 is that $A(t)$ involves only reward parameters; so we do not need to parameterize it if we assume that the reward parameters are specified in advance.
In this case, we only need to parameterize the $q$-function with a single parameter:
\begin{equation*}
q_\phi(t,x,a) = \phi A(t) (a -E(\pi^\phi_t))  - \gamma \mbox{Ent}(\pi^\phi_t),
\end{equation*}
where $\pi^\phi_t(a) \propto e^{-\phi A(t)a/\gamma}$ on $\mathcal{A} = [0,1]$.
A direct computation yields:
\begin{equation*}
l(\phi) = (\rho - \phi) \int_0^T \left(\int_t^T A(s)E(\pi^\phi_s) ds\right)^2 dt.
\end{equation*}

Since $\int_0^T \left(\int_t^T A(s)E(\pi^\phi_s) ds\right)^2 dt \ge c > 0$ on a compact set,
$l$ is dissipative and has a unique stable equilibrium, satisying the conditions (1) and (3).
Other conditions can be verified similarly to (indeed more easily than) Case 1.  
A special case with $r = 0$ gives the following closed-form formula:
\begin{equation*}
l(\phi) = \frac{T^{3}\ell^{2}}{3}
\left(\frac{\gamma}{\phi \ell}-\frac{1}{e^{\phi \ell/\gamma}-1}
\right)^{2}(\rho-\phi).
\end{equation*}
It is easy to see that
$l'(\rho) = -\frac{T^{3}\ell^{2}}{3}
\left(\frac{\gamma}{\rho \ell}-\frac{1}{e^{\rho \ell/\gamma}-1} \right)^2 < 0$;
see Figure \ref{fig:0} for a plot of $l$.
\end{example}

\begin{figure}[h]
\centering
\includegraphics[width=0.5\columnwidth]{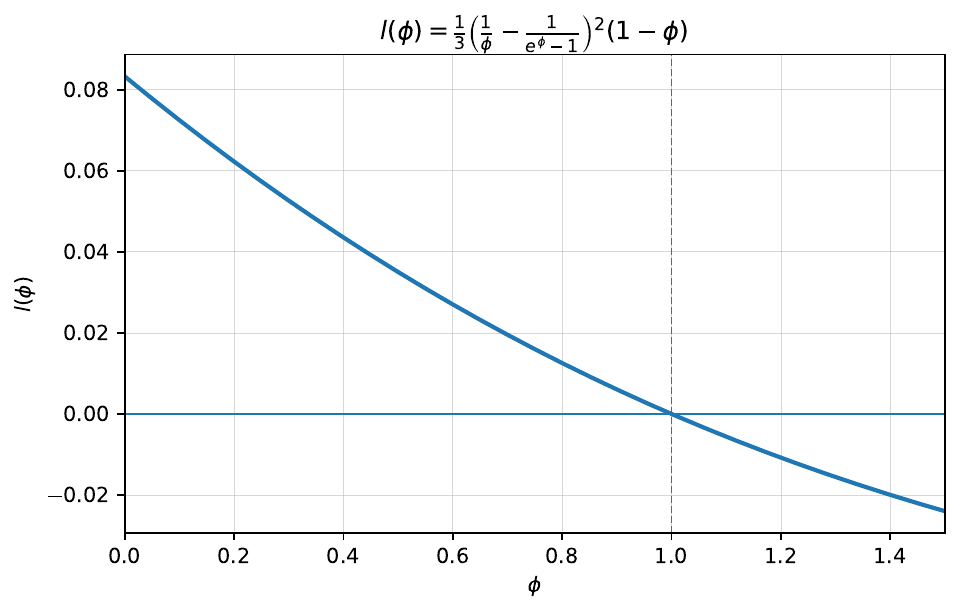}
\caption{Plot of $l(\phi) = \frac{1}{3}( \frac{1}{\phi} - \frac{1}{e^{\phi} - 1})^2(1 - \phi)$ for $\alpha =r = 0$, $\rho = \gamma = \ell =  1$ and $T = 1$.}
\label{fig:0}
\end{figure}

\subsection{$q$-learning}
\label{sc42}

Now we consider the $q$-learning \eqref{eq:SGD0}--\eqref{eq:sample0} by incorporating the policy evaluation part for learning $J^n$.
The idea is similar to the second half of Section \ref{sc41}, which we expand as follows.

The update \eqref{eq:SGD0} can be written as:
\begin{equation}
\label{eq;SGDkk}
(\theta_{n+1}, \phi_{n+1})
= (\theta_n, \phi_n) + (\alpha_{\theta, n}, \alpha_{\phi, n})
\mathcal{H}((\theta_n, \phi_n), \underbrace{(X^{\pi^n}_t), (a^{\pi^n}_t)}_{U_{n+1}}),
\end{equation}
where
\begin{equation}
\mathcal{H}((\theta, \phi), U):=
\left( \int_0^T \frac{\partial J^\theta}{\partial \theta}(t, X_t^{\pi^{\phi}}) G^{\theta, \phi}_{t:T}dt,
 \int_0^T \int_t^T e^{-\beta(s-t)} \frac{\partial q^\phi}{\partial \phi}(s, X_s^{\pi^{\phi}}, a_s^{\pi^{\phi}}) ds G^{\theta, \phi}_{t:T}dt \right),
\end{equation}
with $U = ((X_t^{\pi^\phi}), (a_t^{\pi^\phi}))$,
$\pi^\phi(a \,|\, t,x)\propto \exp\left(\frac{1}{\gamma} q^\phi(t,x,a)\right)$
and $G^{\theta, \phi}_{t:T}: = e^{-\beta (T-t)} h(X_T^{\pi^{\phi}}) - J^\theta(t,X_t^{\pi^{\phi}})
+ \int_t^T e^{-\beta(s-t)} [r(s, X^{\pi^{\phi}}_s, a^{\pi^{\phi}}_s) - q^{\phi}(s, X^{\pi^{\phi}}_s, a^{\pi^{\phi}}_s)] ds$.
Define
\begin{equation}
l(\theta, \phi):= \mathbb{E}(\mathcal{H}(\phi,  \theta), U).
\end{equation}
\quad We need the following assumptions.

\begin{assumption}
\label{assump4}
~
\begin{enumerate}[itemsep = 3 pt]
\item
The ODE $(\theta'(t), \phi'(t)) = l(\theta(t), \phi(t))$ has a unique stable equilibrium $(\theta^*, \phi^*)$.
\item
There exists $C > 0$ such that
$\mathbb{E}(|\mathcal{H}(\theta_n, \phi_n, U_{n+1})|^2\,|\, \theta_n, \phi_n) \le C(1 + |\theta_n|^2 + |\phi_n|^2)$.
\item
There exists $\kappa > 0$ such that
$(\theta - \theta^*, \phi - \phi^*)l(\theta, \phi) \le - \kappa \left( |\phi - \phi^*|^2 + |\theta - \theta^*|^2\right)$ in a neighborhood of $(\theta^*, \phi^*)$.
\item
There exist $\rho > 0$ such that
$|J^{\theta} - J^{\theta'}|_{\infty, R}:=\sup_{t \le T, |x| \le R}  |J^{\theta}(t,x) - J^{\theta'}(t,x)| \le C |\theta - \theta'|^{\rho_\theta}$ for some $C > 0$ (depending on $T$ and $R$), and all $\theta, \theta'$.
\end{enumerate}
\end{assumption}

\quad Set $\Delta_R:= |J^*(\cdot) - J^{\theta^*}(\cdot)|_{\infty, R}$. The following result provides an error estimate of the $q$-learning algorithm,
the proof of which is similar to that of Theorem \ref{thm:main3}.
\begin{theorem}
\label{thm:main5}
Let Assumption \ref{assump4} hold.
Set $\alpha_{\theta,n}, \alpha_{\phi,n} = \frac{A}{n^{\nu} + B}$ for some $\nu \le 1$, $A > \frac{\alpha}{2 \kappa}$ and $B > 0$,
and let $\varepsilon > 0$.
Then there exists $C > 0$ (independent of $n, \varepsilon$) such that
with probability $1 - \varepsilon$,
\begin{equation}
 |J^*(t,x) - J^{\theta_n}(t,x)| \le \Delta_R + \frac{C}{\varepsilon^{\rho_\theta/2}} n^{-\frac{\nu \rho_\theta}{2}},
\end{equation}
for $0 \le t \le T$ and $|x| \le R$.
Consequently, the regret is:
\begin{equation}
\sum_{k = 1}^n |J^*(t,x) - J^{\theta_k}(t,x)| \lesssim \Delta n + \frac{1}{\varepsilon^{\rho_\theta/2}}
\left\{ \begin{array}{rcl}
n^{1 - \frac{\nu \rho_\theta}{2}} & \mbox{if}
& \nu \rho_\theta <2, \\
\ln n & \mbox{if} & \nu \rho_\theta = 2, \\
\mathcal{O}(1) & \mbox{if} & \nu \rho_\theta > 2.
\end{array}\right.
\end{equation}
\end{theorem}

\quad We illustrate Assumption \ref{assump4} with the same linear setting as in Example \ref{ex1}.
\renewcommand{\theexample}{1$^\prime$}
\begin{example}
\em We work under the same problem setting as in Example \ref{ex1}.

\medskip
\noindent
{\em Case 1 (six-parameter policies)}. Assume $\alpha \ne 0$.
Recall the expression of $J(t,x; \pi)$ and $q(t,x,a; \pi)$ obtained earlier. 
We parameterize:
\begin{equation*}
\begin{aligned}
& q_\phi(t,x,a) = F_\phi(t)(a -E(\pi^\phi_t))  - \gamma \mbox{Ent}(\pi^\phi_t),
\quad
\mbox{with } F_\phi(t):= \phi_1 e^{\phi_2(T-t)} + \phi_3,  \\
& J_{\theta, \phi}(t,x,a) = A_\theta(t)x + B_{\phi}(t), \quad \mbox{with }
A_\theta(t):= \theta_1 e^{\theta_2(T-t)} + \theta_3, \\
& \qquad \qquad \qquad \qquad \qquad \qquad \qquad \qquad \,\,
B_{\phi}(t):=\int_t^T [F_\phi(s) E(\pi^{\phi}_s) + \gamma \mbox{Ent}(\pi^{\phi}_s)] ds.
\end{aligned}
\end{equation*}
Recall that for $i \in \{1,2,3\}$, $w_i(t):= \partial_{\phi_i}F_\phi(t)$.
Define $v_i(t):= \partial_{\theta} A_\theta(t)$, with
\begin{equation*}
v_1(t) = e^{\theta_2(T-t)}, \quad v_2(t) = \theta_1(T-t) e^{\theta_2(T-t)}, \quad v_3(t) = 1.
\end{equation*}
Denote by $l_j(\theta, \phi)$ the $j^{th}$ coordinate of $l(\theta, \phi)$ with $j \in \{1,\ldots, 6\}$.
We compute for $j \in \{1,2,3\}$,
\begin{equation*}
\begin{aligned}
&l_j(\theta, \phi) =  \int_0^T v_j(t) \left[(A(t) - A_\theta(t))S_\phi(t) + m_\phi(t) \int_t^T (\rho A(s) - F_\phi(s)) E(\pi^{\phi}_s) ds \right],
 \\
&l_{3+j}(\theta, \phi)=\int_0^T (A(t) - A_\theta(t))m_{\phi}(t) R^\phi_j(t) dt + \int_0^T R^\phi_j(t) \int_t^T (\rho A(s) - F_\phi(s))E(\pi^{\phi}_s) ds dt,
\end{aligned}
\end{equation*}
where $m_{\phi}(t):= \mathbb{E}X_t^{\pi^{\phi}}$, $S_{\phi}(t):= \mathbb{E}(X_t^{\pi^{\phi}})^2$,
and $R^\phi_j(t):= \int_t^T w_j(s) E(\pi^{\phi}_s) ds$.

\quad It is easy to check that the zero of $l$ is 
\begin{equation*}
(\theta^*, \phi^*)\equiv (\theta_1^*, \theta_2^*, \theta_3^*, \phi_1^*, \phi_2^*, \phi_3^*) = \left(\ell + \frac{r}{\alpha},  \alpha, -\frac{r}{\alpha}, \rho\left(\ell + \frac{r}{\alpha}\right),  \alpha, -\frac{\rho r}{\alpha}\right).
\end{equation*}
Similar to Example \ref{ex1}, we assume $\rho\left(\ell + \frac{r}{\alpha}\right) \ne 0$,
and restrict the parameters to a compact set
\begin{equation*}
\mathcal{K}: = \{(\theta, \phi): |\theta_i| \vee |\phi_i|  < C, \,i=1,2,3,\; |\theta_1| \wedge |\phi_1| \ge c , \, \theta_2, \phi_2 \in I\},
\end{equation*}
with $I$ a compact interval containing $\alpha$ but not $0$, for some $c, C > 0$.
In this case, $(A_\theta, F_\phi)$ is injective on $\mathcal{K}$.
Note that $A(t) - A_{\theta^*}(t) = \rho A(t) - F_{\phi^*}(t) = 0$;
so the derivatives of $m_\phi, S_\phi, E(\pi^{\phi}), R^j_\phi$ do not contribute at  $(\theta^*, \phi^*)$.
We get for $i, j \in \{1,2,3\}$,
\begin{equation*}
\begin{aligned}
& \partial_{\theta_i} l_j(\theta^*, \phi^*) = - \int_0^T v_i^*(t) v_j^*(t) S_{\phi^*}(t) dt, \quad
\partial_{\phi_i} l_j(\theta^*, \phi^*) = - \int_0^T v_j^*(t) m_{\phi^*}(t) R_i^{\phi^*}(t) dt, \\
& \partial_{\phi_i} l_{3+j}(\theta^*, \phi^*) = - \int_0^T R_i^{\phi^*}(t) R_j^{\phi^*}(t) dt, \quad
 \partial_{\theta_i} l_{3+j}(\theta^*, \phi^*) = - \int_0^T v_i^*(t) m_{\phi^*}(t) R_j^{\phi^*}(t) dt.
\end{aligned}
\end{equation*}
Writing the Jacobian $Dl(\theta^*, \phi^*) \equiv  -G$, we have for $x, y \in \mathbb{R}^3$,
\begin{equation*}
(x,y)^T G (x,y) = \int_0^T \mathbb{E}\left( \sum_{i = 1}^3 x_i v_i^*(t) X_t^{\pi^{\phi^*}} + \sum_{j = 1}^3 y_j  R_i^{\phi^*}(t)  \right)^2 dt.
\end{equation*}
The quadratic form is zero if and only if
$\sum_i x_i v_i^*(t)  = 0$ and $\sum_j y_j  R_i^{\phi^*}(t) = 0$ almost everywhere (because $\sigma >0$, and hence $X^{\pi^{\phi^*}}$ is non-degenerate).
The independence of $(v_1^*, v_2^*,v_3^*)$ and $(w_1^*, w_2^*,w_3^*)$ implies that $x = y =0$.
As a result, $K$ is positive definite.
Thus, $l$ has a unique stable equilibrium $(\theta^*, \phi^*)$, and is locally dissipative,
leading to the conditions (1) and (3).
The other conditions follow the same arguments as in Example 1.

\medskip
\noindent
{\em Case 2 (one-parameter policies)}. Let $\alpha = 0$.
We parameterize
\begin{equation*}
q_\phi(t,x,a) = \phi A(t) (a -E(\pi^\phi_t))  - \gamma \mbox{Ent}(\pi^\phi_t)
\quad \mbox{and} \quad
J_{\phi}(t,x) = A(t)x + B_\phi(t).
\end{equation*}
That is, we do not need to parameterize the value function $J$ separately.
In this case, we have
\begin{equation*}
l(\phi) = (\rho - \phi) \int_0^T \left(\int_t^T A(s)E(\pi^\phi_s) ds\right)^2 dt.
\end{equation*}
The remainder of the argument is identical to that of Example 1.
\end{example}
\renewcommand{\theexample}{\arabic{example}}

\quad Our ultimate goal is to derive the regret bound of the learned policies $\pi^n(a \,|\, t,x) \propto \exp(q^{\phi_n}(t,x,a))$ in terms of  the {\it original} control problem \eqref{eq:classical} (without the entropy term in the reward functional). To this end,
given a policy $\pi(\cdot)$,
denfine
\begin{equation}
\label{eq:classfin}
\mathring{J}(t,x; \pi):=
\mathbb{E}\left[ \int_t^{T} e^{-\beta (s-t)}r(s, X^{\pi}_s, a^\pi_s) ds + e^{-\beta (T-t)}h(X^\pi_T)\bigg| X^\pi_t = x\right],
\end{equation}
where $(a^\pi_s,\;t\leq s\leq T)$ is sampled from $\pi$ and $\mathbb{E}$ is with respect to both the original probability measure $\mathbb{P}$ and the policy randomization.
The following theorem gives the discrepancy between $\mathring{J}^n(t,x):= \mathring{J}(t,x;\pi^n)$
and $\mathring{J}(t,x)$.
\begin{theorem}
\label{thm:final}
Let the assumptions in Theorem \ref{thm:main2} and Assumption \ref{assump4} hold.
Moreover, assume that $|q^{\phi_*}|_\infty < \infty$.
Set $\alpha_{\phi,n} = \frac{A}{n^{\nu} + B}$ for some $\nu \le 1$, $A > \frac{\alpha}{2 \kappa}$ and $B > 0$,
and let $\varepsilon > 0$.
Then there exists $C > 0$ (independent of $n, \varepsilon$) such that
with probability $1 - \varepsilon$,
\begin{equation}
\label{eq:originalbd}
|\mathring{J}^n(t,x) - \mathring{J}(t,x)| \le C \left(\Delta + \frac{\gamma}{\beta} + \frac{1}{\varepsilon^{\rho_\phi/2}} n^{-\frac{\nu \rho_\phi}{4}} (\ln n)^{\frac{1}{2}} \right) + |J^*(t,x) - \mathring{J}(t,x)|.
\end{equation}
Consequently, the regret is:
\begin{equation}
\begin{aligned}
&\sum_{k = 1}^n |\mathring{J}(t,x) - \mathring{J}^k(t,x)| \\
& \qquad \lesssim \left(\Delta + \frac{\gamma}{\beta} + |J^*(t,x) - \mathring{J}(t,x)| \right) n +   \frac{1}{\varepsilon^{\rho_\phi/2}}
\left\{ \begin{array}{rcl}
n^{1 - \frac{\nu \rho_\phi}{4}} (\ln n)^{\frac{1}{2}} & \mbox{if}
& \nu \rho_\phi <4, \\
(\ln n)^{\frac{3}{2}} & \mbox{if} & \nu \rho_\phi = 4, \\
\mathcal{O}(1) & \mbox{if} & \nu \rho_\phi > 4.
\end{array}\right.
\end{aligned}
\end{equation}
\end{theorem}

\quad A proof of Theorem \ref{thm:final} is deferred to Section \ref{sc43}.
Recall from \eqref{eq:bias} that under additional conditions on the parameters,
$ |J^*(t,x) - \mathring{J}(t,x)| \lesssim \gamma \ln(1/\gamma)$.
Thus, the regret is of order:
\begin{equation*}
\left(C(\gamma) \Delta + \frac{\gamma}{\beta} + \gamma \ln(1/\gamma) \right) n +   \frac{C(\gamma)}{\varepsilon^{\rho_\phi/2}}
\left\{ \begin{array}{rcl}
n^{1 - \frac{\nu \rho_\phi}{4}} (\ln n)^{\frac{1}{2}} & \mbox{if}
& \nu \rho_\phi <4, \\
(\ln n)^{\frac{3}{2}} & \mbox{if} & \nu \rho_\phi = 4, \\
\mathcal{O}(1) & \mbox{if} & \nu \rho_\phi > 4.
\end{array}\right.
\end{equation*}
Note, however, that the dependence on $\gamma$ of the constant $C(\gamma)$ is hard to track.
This is because it is implicitly related to the constants in Assumption \ref{assump4},
and these constants also depend on $\gamma$ in rather obscure ways.
Nevertheless, since the Gibbs policy becomes less regular as $\gamma$ decreases,
we conjecture that it is the main contribution to the dependence $C(\gamma)$,
i.e., $C(\gamma)$ increases rapidly as $\gamma \to 0^+$.
If our conjecture holds, then there is the exploration--stability tradeoff:
a smaller $\gamma$ reduces the entropy bias, but worsens the stability of sampling and learning processes.\footnote{This is in contrast with model-based policy improvement in Section \ref{sc3}, where
a smaller $\gamma$ fosters convergence and, hence, reducing regret.}
A possible approach is to decrease exploration $\gamma = \gamma_n \to \gamma_\star > 0$ (as $n \to \infty$) to mitigate the entropy bias, however,
at the cost of inefficient sampling and learning of policies in each iteration.
The algorithmic implications of the above result is that, for a sufficiently good neural network approximation (so that $\Delta$ is sufficiently small) and a suitably large temperature parameter $\gamma$, the regret of the learned policies applied to the original control problem in the long run is well controlled.

\subsection{Proofs}
\label{sc43}

The proof of Theorem \ref{thm:main2} requires a series of lemmas.
First, similar to Lemma \ref{lem:repJJ}, we establish a representation for $J^n$ given by \eqref{eq:Jngam}.
Recall the definition of $b^n(t,x,a)$ from \eqref{eq:bn}.
Define
\begin{equation}
\label{eq:pnpn}
\pi^n(a \,|\, s,x,z) \propto \exp\left(\frac{1}{\gamma} (b^{n}(s,x,a) \sigma^{-1}(s,x) z+ r(s,x,a)) \right),
\end{equation}
and
\begin{equation}
\label{eq:Gnt}
\begin{aligned}
G^n_s(z,Z):= \int_\mathcal{A} \bigg[ b(s, X_s^{\pi^*}, a) \, \sigma^{-1} & (s, X_s^{\pi^*})Z + r(s,X_s^{\pi^*} ,a) \\
& - \gamma \log \pi^n (a \,|\, s,X_s^{\pi^*}, z) \bigg] \pi^n (a \,|\, s,X_s^{\pi^*}, z) \,da.
\end{aligned}
\end{equation}

\begin{lemma}
\label{lem:repJJ2}
Let Assumptions \ref{assump} and \ref{assump2} hold.
There exists $(\widehat{W}, \widehat{\mathbb{P}})$, with $\widehat{W}$ being $\widehat{P}$-Brownian motion, such that
\begin{equation}
J^n(t,x) = h(X^{\pi^*}_T) + \int_t^T [G^n_s(Z^{n-1}_s, Z^n_s)- \beta Y^n_s]ds - \int_t^T Z^n_s d \widehat{W}_s,
\end{equation}
where $Y^n_s: = J^n(s, X^{\pi^*}_s)$ and $Z^n_s:= (\sigma \nabla J^n)(s, X^{\pi^*}_s)$.
\end{lemma}
\begin{proof}
The proof is the same as Lemma \ref{lem:repJJ}-(i),
noting that $J^n$ satisfies the PDE \eqref{eq:JngamPDE}
with
$\pi^n(a \,|\, s,x) = \pi^n(a \,|\, s, x, (\sigma \nabla J^{n-1})(s,x))$.
\end{proof}

\quad Recall the definition of $G_s(z, Z)$ from \eqref{eq:Gt}.
Next we bound $|G^n_s(z', Z') - G_s(z, Z)|$.
\begin{lemma}
\label{lem:GGn}
Let Assumptions \ref{assump} and \ref{assump2} hold.
Then for $|z|, |Z|, |Z'| \le C$,
there exists $L > 0$ (independent of $\gamma$) such that
\begin{equation}
\label{eq:GnGerr}
\begin{aligned}
|G^n_s(z', Z') - G_s(z, Z)| \le  & L\left( 1 + e^{\frac{L}{\gamma}} + \frac{1}{\gamma} e^{\frac{L}{\gamma}} \right) |z'-z| + L|Z'-Z| \\
& \qquad +
L\left( 1 + \frac{1}{\gamma} e^{\frac{L}{\gamma}}\right) \sup_a |(b - b^{n})(s,X^{\pi^*}_s,a) \sigma^{-1}(s,X^{\pi^*}_s) z'|\;\; a.s.
\end{aligned}
\end{equation}
\end{lemma}
\begin{proof}
Note that
$|G^n_s(z', Z') - G_s(z, Z)| \le |G^n_s(z', Z') - G_s(z', Z')| + |G_s(z', Z') - G_s(z,Z)|$.
By Lemma \ref{lem:sensitivityG},
\begin{equation}
\label{eq:310}
|G_s(z', Z') - G_s(z,Z)| \le L\left( 1 + e^{\frac{L}{\gamma}} + \frac{1}{\gamma} e^{\frac{L}{\gamma}} \right) |z'-z| + L|Z'-Z|\;\; a.s.
\end{equation}
Now we estimate $|G^n_s(z', Z') - G_s(z', Z')|$.
Recall the definitions of $\pi(a \,|\, s,x,z)$ and $\pi^n(a \,|\, s,x,z)$ from \eqref{eq:pipi} and \eqref{eq:pnpn} respectively.
It is easy to see that
\begin{equation}
\label{eq:311}
|G^n_s(z', Z') - G_s(z', Z')| \le 2(\overline{b} \overline{\sigma}C + \overline{r}) E_1 +  E_2,
\end{equation}
where
\begin{align*}
& E_1: = d_{TV}\left(\pi^n(\cdot \,|\, s, X^{\pi^*}_s, z'), \pi(\cdot \,|\, s, X^{\pi^*}_s,z') \right), \\
& E_2: = \gamma \left| \int_\mathcal{A} \log \pi^n(a \,|\, s, X^{\pi^*}_s,z') \, \pi^n(a \,|\, s, X^{\pi^*}_s,z') -
\log \pi(a \,|\, s, X^{\pi^*}_s,z') \, \pi(a \,|\, s, X^{\pi^*}_s,z')da \right|.
\end{align*}
The same argument as in \eqref{eq:TVdiff} yields
\begin{equation}
\label{eq:312}
E_1 \le \frac{1}{\gamma} \exp\left(\frac{2}{\gamma}(\overline{b} \overline{\sigma} C + \overline{r}) \right)
\sup_a |(b - b^{n})(s,X^{\pi^*}_s,a) \sigma^{-1}(s,X^{\pi^*}_s) z'|\;\; a.s.
\end{equation}
Moreover,
\begin{equation}
\label{eq:313}
\begin{aligned}
E_2 & \le \int_\mathcal{A} \left|b^{n}(s, X_s^{\pi^*}, a) \sigma^{-1}(s, X_s^{\pi^*}) z' - b(s, X_s^{\pi^*}, a) \sigma^{-1}(s, X_s^{\pi^*}) z' \right| \, \pi^n(a \,|\, s, X^{\pi^*},z') da \\
& \qquad + \int_\mathcal{A} b(s, X_s^{\pi^*}, a) \sigma^{-1}(s, X_s^{\pi^*}) z' \,  \left|\pi^n(a \,|\, s, X^{\pi^*},z') - \pi(a \,|\, s, X^{\pi^*},z')\right| da \\
& \qquad  + \int_\mathcal{A} |r(t,x,a)| |\pi^n(a \,|\, s, X^{\pi^*},z') - \pi(a \,|\, s, X^{\pi^*},z')| da \\
& \qquad + \gamma \left| \log \left( \frac{\int_\mathcal{A} \exp \left(\frac{1}{\gamma}(b^{n}(s, X_s^{\pi^*}, a) \, \sigma^{-1} (s, X_s^{\pi^*}) z' + r(s, X_s^{\pi^*},a))) \right)da }{\int_\mathcal{A} \exp \left(\frac{1}{\gamma}(b(s, X_s^{\pi^*}, a) \, \sigma^{-1} (s, X_s^{\pi^*}) z' + r(s, X_s^{\pi^*},a))) \right)da} \right) \right| \\
& \le 2 \sup_a |(b - b^{n})(s,X^{\pi^*}_s,a) \sigma^{-1}(s,X^{\pi^*}_s) z'| + 2 (\overline{b}\overline{\sigma} C + \overline{r}) E_1  \\
& \le 2  \left(1 + \frac{\overline{b} \overline{\sigma} C + \overline{r}}{\gamma} \exp \left(\frac{2}{\gamma} (\overline{b} \overline{\sigma} C + \overline{r}) \right)\right) \sup_a |(b - b^{n})(s,X^{\pi^*}_s,a) \sigma^{-1}(s,X^{\pi^*}_s) z'|\;\; a.s.
\end{aligned}
\end{equation}
By \eqref{eq:311}, \eqref{eq:312} and \eqref{eq:313}, we get
\begin{equation}
\label{eq:314}
\begin{aligned}
& |G^n_s(z', Z') - G_s(z', Z')| \\
& \qquad \quad \le 2  \left(1 + \frac{2 (\overline{b} \overline{\sigma} C + \overline{r})}{\gamma} \exp \left( \frac{2}{\gamma} (\overline{b} \overline{\sigma} C + \overline{r})\right) \right) \sup_a |(b - b^{n})(s,X^{\pi^*}_s,a) \sigma^{-1}(s,X^{\pi^*}_s) z'|\;\; a.s.
\end{aligned}
\end{equation}
Combining \eqref{eq:310} and \eqref{eq:314} yields \eqref{eq:GnGerr}.
\end{proof}

\quad We also need the following variant of Lemma \ref{lem:BSDE} on the BSDE.
\begin{lemma}
\label{lem:BSDE2}
For $i = 1,2$, let $F^i_t(y,z,Z)$ be a measurable function of $(t,y,z,Z,\omega)$ such that $(F^i_t(y,z,Z), \, 0 \le t \le T)$ is predictable for any fixed $(y,z,Z)$ and
$(F^i_t(0,0,0), \, 0 \le t \le T) \in \mathbb{H}^0_0$,
and there exist $M_1, M_2, M_3 > 0$,
\begin{equation*}
|F^i_t(y',z',Z') - F^i_t(y,z,Z)| \le M_1 |y'-y| + M_2 |z'-z| + M_3 |Z' - Z| \quad a.s.,
\end{equation*}
and that
$\delta: = \sup_{t, \omega} |F^1_t - F^2_t|_\infty < \infty$.
Further, let $z^i \in \mathbb{H}^0_0$ and $(Y^i, Z^i)$ be the unique solution to
\begin{equation*}
Y_t = \xi + \int_t^T F^i_t(Y_s, z^i_s, Z_s) ds - \int_t^T Z_s dW_s, \quad 0 \le t \le T \quad a.s.
\end{equation*}
Fixing $\eta \in (0,1)$, we have for $\theta \ge M_1 + (1+\eta^{-1}) (M_2+M_3 + \delta)^2$,
\begin{equation*}
e^{\theta t} \mathbb{E}|Y^1_t - Y^2_t|^2  + |Z^1 - Z^2|^2_{\mathbb{H}^\theta_t} \le \eta (|z^1 - z^2|^2_{\mathbb{H}^\theta_t} + \delta).
\end{equation*}
\end{lemma}
\begin{proof}
It suffices to note that $|F_t^1(y,z,Z) - F_t^2(y',z',Z') | \le \delta + M_1 |y' - y| + M_2 |z'-z| + M_3 |Z' - Z|$ a.s.,
and the same argument as in \cite[Lemma A.5]{KSS20} permits to conclude.
\end{proof}

\quad Now we give the proof of Theorem \ref{thm:main2}.
\begin{proof} (of Theorem \ref{thm:main2})
Let $F_s(Y,z,Z): = G_s(z,Z) - \beta Y$ and $F^n_s(Y,z,Z): = G^n_s(z,Z) - \beta Y$.
By Lemma \ref{lem:repJJ2}, we have
$J^n(t,x) = Y^n_t$ and $J^*(t,x) = Y^{t,x}_t$,
where
\begin{equation*}
\begin{aligned}
& Y^n_s = h(X^{\pi^*_\gamma}_T) + \int_s^T F^n_u(Y^n_u, Z^{n-1}_u, Z^n_u) \, du - \int_s^T Z^n_u d\widehat{W}_u, \\
& Y^{t,x}_s = h(X^{\pi^*_\gamma}_T)  + \int_s^T F_u(Y^{t,x}_u, Z^{t,x}_u, Z^{t,x}_u) \, du - \int_s^T Z^{t,x}_u d\widehat{W}_u,
\end{aligned}
\end{equation*}
with $Z^n_s= \left(\sigma \nabla J^n \right)(s, X^{\pi^*}_s)$ and
$Z^{t,x}_s= \left(\sigma \nabla J^* \right)(s, X^{\pi^*}_s)$.
By Lemma \ref{lem:GGn}
and the assumption that $|\nabla J^n| \le M$ for all $n$,
we get
\begin{equation*}
\begin{aligned}
& |F^n_u(Y^n_u, Z^{n-1}_u, Z^n_u) - F_u(Y^{t,x}_u, Z^{t,x}_u, Z^{t,x}_u)| \\
&  \le \beta |Y^n_u - Y^{t,x}_u| +  L \left( 1+ e^{\frac{L}{\gamma}} + \frac{1}{\gamma} e^{\frac{L}{\gamma}} \right) |Z^{t,x}_u- Z^{n-1}_u| \\
& \qquad \qquad \qquad \qquad \qquad \qquad \qquad \qquad+ L|Z^{t,x}_u-Z^n_u|
+ L\left(1 + \frac{1}{\gamma} e^{\frac{L}{\gamma}} \right) |q(\cdot;\pi^{n-1}) - q^{\phi_n}|_\infty.
\end{aligned}
\end{equation*}
Applying Lemma \ref{lem:BSDE2} with
$\delta = L\left(1 + \frac{1}{\gamma} e^{\frac{L}{\gamma}} \right) |q(\cdot ;\pi^{n-1}) - q^{\phi_n}|_\infty$,
$M_1 = \beta$, $M_2 = L \left( 1+ e^{\frac{L}{\gamma}} + \frac{1}{\gamma} e^{\frac{L}{\gamma}} \right)$
and $M_3 = L$,
we have for $\theta \ge \beta + (1+\eta^{-1}) L^2\left[ 1+ e^{\frac{L}{\gamma}} + \frac{1}{\gamma} e^{\frac{L}{\gamma}} +
L \left( 1 + \frac{1}{\gamma} e^{\frac{L}{\gamma}}\right) |q(\cdot;\pi^{n-1}) - q^{\phi_n}|_\infty\right]^2$,
\begin{equation*}
e^{\theta t} \widehat{\mathbb{E}}|Y^{t,x}_t - Y^n_t|^2  + |Z^{t,x} - Z^n|^2_{\widehat{\mathbb{H}}^\theta_t} \le \eta \left[ |Z^{t,x} - Z^{n-1}|^2_{\widehat{\mathbb{H}}^\theta_t} + L \left( 1 + \frac{1}{\gamma} e^{\frac{L}{\gamma}}\right) |q(\cdot;\pi^{n-1}) - q^{\phi_n}|_\infty\right].
\end{equation*}
Thus,
$|Z^{t,x} - Z^n|^2_{\widehat{\mathbb{H}}^\theta_t} \le \eta^n  |Z^{t,x} - Z^0|^2_{\widehat{\mathbb{H}}^\theta_t} + L \left( 1 + \frac{1}{\gamma} e^{\frac{L}{\gamma}}\right) \sum_{k=1}^{n} \eta^{n+1-k} |q(\cdot;\pi^k) - q^{\phi_{k+1}}|_\infty$, and
\begin{equation*}
\begin{aligned}
|J^*_\gamma(t,x) - J^n_\gamma(t,x)|^2  &\le \eta^n \int_t^T e^{\theta (T-t)} |\sigma(s, X^{\pi^*_\gamma}_s)|^2 | \nabla J^*(s, X^{\pi^*_\gamma}_s) - \nabla J^0(s, X^{\pi^*_\gamma}_s)|^2 ds \\
& \qquad \qquad \qquad + L \left( 1 + \frac{1}{\gamma} e^{\frac{L}{\gamma}}\right) e^{-\theta t} \sum_{k=1}^{n} \eta^{n+1-k} |q(\cdot;\pi^k) - q^{\phi_{k+1}}|_\infty.
\end{aligned}
\end{equation*}
This yields the bound \eqref{eq:cv2}.
\end{proof}

\quad We proceed to proving Theorem \ref{thm:main3}.
\begin{proof} (of Theorem \ref{thm:main3})
By the argument of \cite[Theorem 22]{BMP90},
Assumption \ref{assump3} (1)--(3) and the condition on $\alpha_{\phi,n}$ imply
\begin{equation*}
\mathbb{E}|\phi^n - \phi^*|^2 \le C n^{-\nu} \quad \mbox{for some } C > 0 \mbox{ (independent of } n).
\end{equation*}
Thus,
 $|\phi^n - \phi^*| \le C\varepsilon^{-\frac{1}{2}} n^{-\frac{\nu}{2}}$ with a probability of at least $1 - \varepsilon$.
Next by Assumption \ref{assump3} (4), we have:
\begin{equation}
\label{eq:425}
|q(\cdot; \pi^*) - q^{\phi_n}|_\infty \le \Delta + C \varepsilon^{-\frac{\rho_\phi}{2}}n^{-\frac{\nu \rho_\phi}{2}}.
\end{equation}
It is easy to deduce that
$d_{TV}(\pi^n, \pi^*) \le C (\Delta + \varepsilon^{-\frac{\rho_\phi}{2}}n^{-\frac{\nu \rho_\phi}{2}})$,
since $\pi^n(a \,|\, t,x) \propto \exp(\frac{1}{\gamma}q^{\phi_n}(t,x,a))$
and $\pi^*(a \,|\, t,x) \propto \exp(\frac{1}{\gamma}q(t,x,a; \pi^*))$.
By Assumption \ref{assump3} (5), we get for $\Delta$ sufficiently small,
\begin{equation}
\label{eq:426}
|q(\cdot; \pi^n) - q(\cdot; \pi^*)|_\infty \le C(\Delta + \varepsilon^{-\frac{\rho_\phi}{2}}n^{-\frac{\nu \rho_\phi}{2}}).
\end{equation}
Combining \eqref{eq:425} and \eqref{eq:426} gives
$|q(\cdot; \pi^n) - q^{\phi_{n+1}}|_\infty \le C(\Delta + \varepsilon^{-\frac{\rho_\phi}{2}}n^{-\frac{\nu \rho_\phi}{2}})$.
Applying Theorem \ref{thm:main2} and \eqref{eq:multbound} yields the bound \eqref{eq:imp}.
\end{proof}

\quad Finally, we prove Theorem  \ref{thm:final}.
\begin{proof} (of Theorem \ref{thm:final})
Note that
\begin{equation}
\label{eq:trianglefin}
\begin{aligned}
|\mathring{J}^n(t,x) - \mathring{J}(t,x)|
& = |\mathring{J}^n(t,x) -J(t,x; \pi^n) + J(t,x; \pi^n) - J^*(t,x) + J^*(t,x) - \mathring{J}(t,x)| \\
& \le |\mathring{J}^n(t,x) -J(t,x; \pi^n)| + |J(t,x; \pi^n) - J^*(t,x)| + |J^*(t,x) - \mathring{J}(t,x)|.
\end{aligned}
\end{equation}
It follows from Theorem \ref{thm:main3} that
\begin{equation}
\label{eq:fr44}
|J(t,x; \pi^n) - J^*(t,x)| \le C \left(\Delta + \frac{1}{\varepsilon^{\rho_\phi/2}} n^{-\frac{\nu \rho_\phi}{4}} (\ln n)^{\frac{1}{2}} \right).
\end{equation}
Moreover,
$$|\mathring{J}^n(t,x) -J(t,x; \pi^n)| \le \gamma \mathbb{E}\left[\int_t^T e^{-\beta(s-t)} \int_{\mathcal{A}} |\log \pi^n(a \,|\, s,X^{\pi^n}_s) | \pi^n(a \,|\, s,X^{\pi^n}_s) da \bigg| X^{\pi^n}_t =x \right].$$
Note that
$|\log \pi^n(a \,|\, t,x) | \le q^{\phi_n}(t,x,a) + |\log \int_{\mathcal{A}} \exp(q^{\phi_n}(t,x,a)) da|$.
By Assumption \ref{assump4} (4) and that $|q^{\phi_{*}}|_\infty < \infty$,
we conclude that $|\log \pi^n(a \,|\, t,x) |$ is uniformly bounded.
As a result,
\begin{equation}
\label{eq:firstfinal}
|\mathring{J}^n(t,x) -J(t,x; \pi^n)| \le \frac{C \gamma}{\beta} \quad \mbox{for some } C > 0.
\end{equation}
Combining \eqref{eq:trianglefin}, \eqref{eq:fr44} and \eqref{eq:firstfinal} yields \eqref{eq:originalbd}.
\end{proof}

\section{Conclusion}
\label{sc5}

\quad This paper studies convergence of various RL algorithms for controlled diffusions.
We provide the convergence rate and the regret of
exploratory policy iteration, semi-$q$-learning,
and $q$-learning.
The tools that we develop in this paper encompass
stochastic control, partial differential equations and probability theory (BSDEs in particular).

\quad As continuous RL is still in the early innings and to our best knowledge this paper is the first to study convergence and regret of $q$-learning, there
are numerous open questions.
First, one can relax some technical assumptions,
e.g. the uniform boundedness of $\nabla J^n$ in Theorem \ref{thm:main2},
and provide general sufficient conditions on the model parameters
to ensure Assumptions \ref{assump3} and \ref{assump4} (1)--(3),(5). In particular, it is interesting to inquire if these assumptions hold in the important class of linear--quadratic (LQ) controls.
Second, we only exercise control (and hence exploration) on the drift,
and it is important and curious, as well as challenging, to consider control-dependent diffusion coefficients.
Finally, it is intriguing  to investigate whether the established convergence rates and regrets
are optimal. For that, considering exploratory annealing is a promising direction.

\bigskip
{\bf Acknowledgement}:
We thank the Action Editor and three reviewers for careful reading and constructive suggestions
that help us improve the paper substantially.
Tang gratefully acknowledges financial support through NSF CAREER Award DMS-2538791, and the Tang
Family Assistant Professorship.
Zhou gratefully acknowledges financial supports through the Nie Center for Intelligent Asset Management at Columbia University.
Tang and Zhou receive support from the Columbia Innovation Hub grant,
and are part of a Columbia-CityU/HK collaborative project that is supported by InnoHK Initiative, The Government of the HKSAR and the AIFT Lab.






\bibliographystyle{abbrv}
\bibliography{unique}

\end{document}